\documentclass{article}
\addtocontents{toc}{\protect\setcounter{tocdepth}{0}}
\usepackage[preprint]{neurips_2026}

\usepackage[utf8]{inputenc} 
\usepackage[T1]{fontenc}    

\usepackage[
  colorlinks=true,
  linkcolor={rgb,1:red,0;green,0;blue,0.5},
  citecolor={rgb,1:red,0;green,0.078;blue,0.451},
  urlcolor={rgb,1:red,0;green,0;blue,0.5}
]{hyperref}
\setcitestyle{round}
\usepackage{url}            
\usepackage{booktabs}       
\usepackage{amsfonts}       
\usepackage{nicefrac}       
\usepackage{microtype}      
\usepackage{xcolor}         
\usepackage{amsmath,amsthm}
\usepackage{multirow}
\usepackage{threeparttable}
\usepackage{algorithm}
\usepackage{algorithmic}
\usepackage{graphicx}
\usepackage{wrapfig}
\usepackage{enumitem}
\usepackage{caption}
\newtheorem{theorem}{Theorem}
\newtheorem{lemma}{Lemma}
\newtheorem{assumption}{Assumption}

\newtheorem{proposition}{Proposition}

\newtheorem{corollary}{Corollary}
\DeclareMathOperator*{\argmax}{arg\,max}

\title{Efficient LLM Adversarial Training via Low-Rank Defense and Circuit-Guided Surrogates}

\author{%
  Weiyi He \quad Yuping Lin \quad Jiliang Tang \quad Yue Xing \\
  \\
  Michigan State University \\
  \\
  \texttt{\{heweiyi, linyupin, tangjili, xingyue1\}@msu.edu} \\
}
\begin{document}

\maketitle

\begin{abstract}
Adversarial training is one of the most effective defenses against adversarial attacks, yet the computational cost remains prohibitive at modern scales, especially for large language models (LLMs). While existing mitigation strategies, e.g., latent adversarial training (LAT), have been developed, they still incur a high computational cost. In this work, focusing on LLM-based classifiers, we investigate computation-efficient strategies to speed up LAT from two complementary perspectives: (1) Defense-side optimization: We explore representation fine-tuning (ReFT) within LAT, and reveal a potential issue if there is a mismatch on which tokens to apply ReFT and the attack. (2) Attack-side optimization: When computing adversarial attacks in each LAT iteration, we extract only the relevant circuits from the LLM to construct a lightweight surrogate model, avoiding the computation in the forward-backward passes through the full model during the attack generation. For both perspectives, we provide theoretical justifications and numerical evidence to illustrate the effectiveness of the proposed strategies. Compared to standard LAT with full fine-tuning, our method on average reduces per-step adversarial-training FLOPs by 29.6\% while requiring only 0.0118\% trainable parameters, at a moderate cost in robustness.
\end{abstract}

\section{Introduction}
Large language models (LLMs) have achieved remarkable performance across diverse tasks, yet remain critically vulnerable to adversarial attacks that manipulate model outputs via crafted prompts \citep{zou2023universal, yu2023gptfuzzer, andriushchenko2024jailbreaking}. These vulnerabilities pose serious risks for real-world deployments, ranging from security-critical classifiers for tasks such as content moderation and filtering to agentic systems \citep{inan2023llama, howe2024scaling}, making robust defense mechanisms essential. For example, recent studies have shown that small input perturbations can alter the decisions of widely deployed safety classifiers \citep{achara2025watching, hackett2025bypassing}, while prompt injection has been exploited to induce unsafe downstream actions in AI coding agents \citep{verge_cline_2026}.

To ensure the robustness of LLMs, one of the most effective defenses is adversarial training \citep{xing2021algorithmic,zhao2024adversarial}, which trains LLMs on adversarial examples to help them better defend against adversarial inputs. However, scaling adversarial training to modern LLMs is computationally infeasible. The cost arises from two sources: (i) updating a large number of model parameters during training, and (ii) repeatedly running the full model to generate adversarial examples via iterative optimization in the token space \citep{zou2023universal,howe2024scaling}. Recent works have proposed latent adversarial training (LAT) \citep{xhonneux2024efficient, sheshadri2024latent, casper2024defending}, which perturbs continuous hidden representations rather than optimizing directly in the discrete token space. This makes the optimization easier than token-level adversarial example generation \citep{howe2024scaling}, but the computational bottleneck remains significant: running the inner attack with Projected Gradient Descent (PGD) \citep{madry2017towards} for 8 steps alone accounts for approximately two-thirds of the computation in LAT under the Kaplan FLOPs estimate \citep{kaplan2020scaling}.

While parameter-efficient training strategies such as LoRA and representation fine-tuning (ReFT) \citep{wu2024reft, ren2025general, zeng2025towards} have been widely adopted, significant gaps persist when applying them in the context of LAT. \textbf{(G1)} On the defense side (i.e., updating model parameters to enhance robustness), ReFT requires substantially fewer trainable parameters than LoRA, making it an attractive candidate for efficient adversarial training. However, standard ReFT often applies interventions on task-relevant positions, sometimes only the last token at a specific hidden layer (usually not the final layer). Since adversarial attacks can perturb multiple tokens, it remains unclear whether such position-limited interventions can effectively defend multi-token attacks. \textbf{(G2)} On the attack side, while ReFT and other parameter-efficient methods reduce the cost of model parameter updates, they do not address the computational burden of attack generation itself: computing adversarial perturbations still requires iterative forward-backward passes through the full model at each training step. 

To address these gaps, we study \textit{adversarial robustness in classification settings} with a focus on token-level suffix attacks, following \cite{howe2024scaling}. We systematically explore strategies for reducing LAT's computation cost from both defense and attack perspectives, both theoretically and empirically. Our contributions are as follows:

\begin{itemize}[leftmargin=*]
\item \textbf{Defense side:} We introduce LAT-ReFT, a parameter-efficient latent defense that integrates ReFT into LAT. To analyze \textbf{(G1)}, as in Figure~\ref{fig:intro_placement_motivation}, when the interventions are added on a suffix window of the attacked tokens, the attack success rate is much lower. We further theoretically prove that single-token defense is insufficient: attention mechanisms allow attacks from earlier suffix positions to leak into the defended output (Theorem~\ref{thm:causal_bypass}). We then derive conditions under which multi-token suffix window ReFT can provably suppress this leakage (Theorem~\ref{thm:seq_shared_defense}). Empirical results demonstrate the effectiveness of the proposed algorithm in reducing computation costs.

\item \textbf{Attack side:} Inspired by mechanistic interpretability research showing that model behaviors are determined by sparse circuits \citep{wang2022interpretability, elhage2022toy, hanna2024have, chen2024towards}, we develop a circuit-guided surrogate attack strategy to overcome \textbf{(G2)}: generating attacks on a pruned surrogate model and transferring them to the full model. The key challenge is balancing efficiency and transferability: naive pruning accelerates attack generation but severely degrades attack effectiveness. Inspired by previous work \citep{molchanov2016pruning,molchanov2019importance,syed2024attribution}, we develop an activation-gradient importance scoring, ActGrad, to identify attack-critical neurons, and provide theoretical support in Theorem~\ref{thm:surrogate}. Empirical results show that pruning 25\% of MLP neurons preserves attack-relevant geometry while reducing the per-step training FLOPs by about 20\%. Combined with our defense-side strategy, we achieve an overall computational reduction of 29.6\% on average compared to standard LAT with full fine-tuning.

\item \textbf{Unified insight:} Despite operating through different mechanisms, the defense and attack components reveal a common layerwise pattern: effective ReFT interventions should be placed in early or middle layers rather than very late layers, and the attack surrogate likewise retains most of its important neurons from early and middle layers. This observation is broadly consistent with recent interpretability evidence suggesting that late layers in Transformers often have limited functional complexity \citep{skean2025layer,queipo2025attention}, and it provides practical guidance for efficient adversarial training design. 
\end{itemize}

\section{Related Work}

\paragraph{Adversarial Training for LLMs.}
Adversarial training (AT) is a widely used approach for improving model robustness by optimizing performance under worst-case perturbations \citep{madry2017towards,shafahi2019adversarial,wong2020fast,andriushchenko2020understanding}. In the context of LLMs, AT typically involves generating adversarial prompts in the discrete token space, which incurs high computational cost \citep{ebrahimi2018hotflip,zou2023universal,howe2024scaling}. To address this challenge, recent work has proposed Continuous Adversarial Training (CAT), where adversarial perturbations are applied in the embedding space \citep{xhonneux2024efficient, dekany2025mixat}. This allows efficient gradient-based optimization, reducing the cost of adversarial example generation. Subsequent work further extends this framework to attacking the latent representations in the intermediate LLM layers \citep{sheshadri2024latent, casper2024defending}. Recent theoretical analyses also provide justification for the effectiveness of such continuous perturbations \citep{fu2026understanding}.

\paragraph{Representation Engineering.}
Representation engineering studies how model behavior can be controlled by operating directly on internal activations, rather than by updating all model parameters \citep{zou2023representation}. Prior work has shown that hidden representations in LLMs encode rich semantic and behavioral information, and modifying these representations can effectively steer model outputs \citep{turner2023steering, zheng2024prompt, lin2024towards, arditi2024refusal}. Within this line of work, representation fine-tuning (ReFT) provides a parameter-efficient adaptation mechanism by learning localized low-rank interventions in activation space \citep{wu2024reft}. Subsequent studies further show that such representation-level updates can serve as an effective alternative to full-parameter fine-tuning \citep{ren2025general, zeng2025towards}. This is closely related to our setting, since latent adversarial defense likewise operates on internal representations and thus naturally benefits from such interventions.
\begin{wrapfigure}[]{r}{0.5\columnwidth} 
    \centering
    \vspace{-0.5em}  
    \includegraphics[width=0.5\columnwidth]{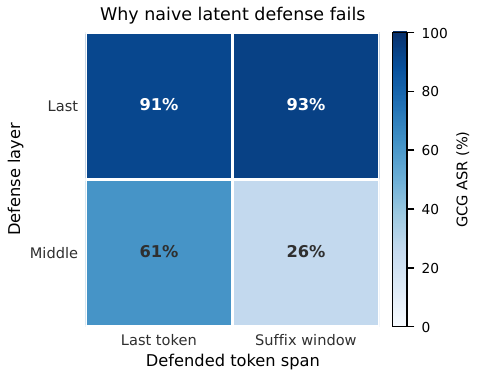}
    \caption{A naive latent defense that protects only the last token at the last layer is ineffective. Robustness improves substantially when the defense covers a suffix window and is moved away from the end of the model.}
    \vspace{-3em}  
    \label{fig:intro_placement_motivation}
\end{wrapfigure}
\paragraph{Mechanistic Interpretability and Circuits.}
Mechanistic interpretability aims to identify the internal components and computational pathways responsible for specific model behaviors \citep{chen2024towards, lee2024mechanistic,sharkey2025open}. A recurring finding in this literature is that many behaviors in LLMs are mediated by sparse circuits or localized subnetworks, rather than being uniformly distributed across all parameters \citep{wang2022interpretability, hanna2024have, patel2025alignment}. To study such locality, prior work has developed methods for localizing and validating important components \citep{vig2020investigating, frantar2023sparsegpt}.  In our work, these findings motivate reducing the computation used in adversarial optimization by focusing on the model components that are most relevant to the expression of adversarial behaviors.

\section{Preliminaries}

\paragraph{Transformers.}
A decoder-only Transformer \citep{vaswani2017attention} maps an input token sequence $x=[x_1,\dots,x_T]$ to output logits. Let $\mathbf{h}_t^l \in \mathbb{R}^d$ denote the hidden representation of token $t$ at layer $l \in [L_\mathrm{all}]$. Each layer contains a self-attention block and an MLP block with residual connections. For notational simplicity, we write the layer update as
\(
\mathbf{h}_t^l = \mathbf{h}_t^{l-1} + \mathbf{a}_t^l + \mathbf{m}_t^l,
\)
where $\mathbf{a}_t^l$ and $\mathbf{m}_t^l$ denote the attention and MLP updates, respectively.

\paragraph{Latent Adversarial Training (LAT).}
Let $f_\theta(x)$ denote the model logits for input $x$, parameterized by $\theta$, and let $\ell(f_\theta(x),y)$ be the classification loss for label $y$. Standard adversarial training seeks parameters that are robust to worst-case perturbations:
\(
\min_\theta \; \mathbb{E}_{(x,y)}
\left[
\max_{x' \in \mathcal{B}(x,\varepsilon)} \ell(f_\theta(x'), y)
\right],
\)
where $\mathcal{B}(x,\varepsilon)$ denotes the set of perturbations around $x$ under a perturbation budget $\varepsilon$.

To save the cost of calculating the discrete attacked tokens, LAT \citep{sheshadri2024latent, casper2024defending} instead performs the inner maximization in a continuous hidden space. Let $\mathbf{H}_l(x;\theta) \in \mathbb{R}^{T \times d}$ denote the hidden-state tensor at layer $l$ for the input $x$. Rather than perturbing the discrete input directly, LAT optimizes over hidden representations in a neighborhood of $\mathbf{H}_l(x;\theta)$:
\begin{equation}
\min_\theta \; \mathbb{E}_{(x,y)}
\left[
\ell(f_\theta(x), y)
+
\lambda_{\mathrm{adv}}
\max_{\widetilde{\mathbf H}\in \mathcal{B}(\mathbf{H}_l(x;\theta),\varepsilon)}
\ell\bigl(f_\theta(x;\widetilde{\mathbf H}), y\bigr)
\right].
\end{equation}
Here, $\mathcal{B}(\mathbf{H}_l(x;\theta),\varepsilon)$ denotes the neighborhood in hidden space, and $f_\theta(x;\widetilde{\mathbf H})$ denotes the model output when the hidden representation at layer $l$ is replaced by $\widetilde{\mathbf H}$. Since the inner optimization is carried out in continuous representation space, it can be efficiently approximated using PGD. 

\paragraph{Representation Fine-tuning (ReFT).}
ReFT \citep{wu2024reft} restricts model updates to a low-rank subspace of the hidden representation. Let $\mathbf{R} \in \mathbb{R}^{r \times d}$ be a matrix with orthonormal rows, where $r \le d$ is the subspace rank. Given trainable parameters $\mathbf{W} \in \mathbb{R}^{r \times d}$ and $\mathbf{b} \in \mathbb{R}^r$, the ReFT operator applied to a hidden state $\mathbf{h} \in \mathbb{R}^d$ is
\(
\Phi_{\mathbf{R}}(\mathbf{h}) = \mathbf{h} + \mathbf{R}^\top(\mathbf{W}\mathbf{h} + \mathbf{b} - \mathbf{R}\mathbf{h}).
\)

\section{Method}

Our goal is to make adversarial training for LLMs both effective and efficient. On the defense side, we introduce LAT-ReFT, a parameter-efficient defense that restricts adaptation to a low-rank subspace of the hidden representation. On the attack side, we introduce a circuit-guided surrogate that reduces the cost of inner-loop adversarial optimization while being transferable to the full defended model.

\subsection{Defense: LAT-ReFT}
\label{sec:lat-reft}

We integrate ReFT into LAT to obtain a parameter-efficient latent defense by freezing the backbone and training only a low-rank intervention in representation space. Let $l_{\mathrm a}$ denote the attack layer and $l_{\mathrm r}$ the defense layer, with $l_{\mathrm r}\ge l_{\mathrm a}$. Let $\mathcal T_{\mathrm{LAT}}$ denote the token positions where latent perturbations are allowed. Consider perturbations $\Delta_{\mathcal T}\in\mathbb R^{T\times d}$, supported only on $\mathcal T_{\mathrm{LAT}}$, i.e., $(\Delta_{\mathcal T})_{t,:}=0$ for all $t\notin\mathcal T_{\mathrm{LAT}}$. Thus, the attack is restricted to the selected token positions at layer $l_{\mathrm a}$, while ReFT is applied at layer $l_{\mathrm r}$ to correct the resulting adversarial features.
Let $f_{\mathbf R,\mathbf W,\mathbf b}(x;\widetilde{\mathbf H})$ denote the model output when the hidden representation at layer $l_{\mathrm a}$ is replaced by $\widetilde{\mathbf H}$ and the ReFT module parameterized by $(\mathbf R,\mathbf W,\mathbf b)$ is applied at layer $l_{\mathrm r}$. The defense objective is
\begin{equation}
\label{Eq:objective}
\min_{\mathbf R,\mathbf W,\mathbf b}
\;
\mathbb E_{(x,y)}
\left[
\ell\!\left(f_{\mathbf R,\mathbf W,\mathbf b}(x),y\right)
+
\lambda_{\mathrm{adv}}
\max_{\widetilde{\mathbf H}\in
\mathcal B(\mathbf H_{l_{\mathrm a}}(x;\theta),\varepsilon;\mathcal T_{\mathrm{LAT}})}
\ell\!\left(
f_{\mathbf R,\mathbf W,\mathbf b}(x;\widetilde{\mathbf H}),
y
\right)
\right].
\end{equation}
Here, $\mathcal B(\mathbf H_{l_{\mathrm a}}(x;\theta),\varepsilon;\mathcal T_{\mathrm{LAT}})$ denotes the hidden-space neighborhood obtained by restricting perturbations to the token positions in $\mathcal T_{\mathrm{LAT}}$. 

A central design issue for ReFT-based latent defense is \emph{where} to place the intervention. As illustrated in Figure~\ref{fig:intro_placement_motivation}, this choice has two parts: (1) selecting the defended suffix window and (2) selecting the intervention layer. We discuss these two design choices in turn below.

\paragraph{Suffix-window defense.}
A ReFT intervention applied only at the last token is not sufficient in our setting. Empirically, restricting the defense to a single position leaves the model vulnerable to suffix-style attacks, indicating that the intervention is too local. 
We therefore defend a contiguous suffix window rather than a single token:
\(
\mathcal T_{\mathrm{LAT}}=\{T-L+1,\dots,T\}.
\)
We provide theoretical support for this design in Section~\ref{sec:theory}. 
To supplement, we also provide results for prefix/infix attacks and present that a suffix window can still help defend against those attacks; see Appendix~\ref{app:prefix_attack}.

\paragraph{Defense layer.}
Prior latent-defense work typically treats layer choice empirically \citep{sheshadri2024latent, casper2024defending}, while existing analyses suggest that Transformer layers can be roughly grouped into early, middle, and late stages with different roles \citep{skean2025layer}. Guided by both perspectives, we compare intervention layers from these regions in Section~\ref{sec:exp_design}. Our results show that late-layer intervention is generally ineffective, and ReFT should be away from the end of the model.

\subsection{Circuit-Guided Surrogate for Attack Generation}
\label{sec:surrogate}

Another problem in LAT is the cost of repeatedly solving the inner maximization in Eq.~(\ref{Eq:objective}). Even though latent optimization is cheaper than discrete token-space search, generating adversarial perturbations still requires many PGD steps with repeated forward--backward passes through a large model. To reduce this cost, we introduce a circuit-guided surrogate. 

Specifically, we use a surrogate model for inner-loop attack generation: the latent perturbation is optimized on the surrogate and then applied to the full defended model for training. We construct the surrogate by compressing the MLP blocks of the full model. MLP blocks are a natural target because they contain about two-thirds of a Transformer's parameters and account for much computation. They are also straightforward to compress in a structured way, as their expanded hidden layer is organized neuron-wise, so we can retain only a subset of neurons and prune the rest \citep{geva2021transformer}.
To select neurons to retain, we assign each MLP neuron an importance score inspired by attribution analyses of internal model components, which use activation- and gradient-based signals to estimate component importance \citep{syed2024attribution}. Concretely, for neuron $j$ in layer $l$, we define
\begin{equation}
s_{l,j}
=
\sum_{(x,y)\in\mathcal D}
\sum_t
\left|
a_{t,j}^{(l)}(x)\,
\frac{\partial \ell(f_\theta(x),y)}{\partial a_{t,j}^{(l)}}
\right|,
\end{equation}
where $a_{t,j}^{(l)}(x)$ denotes the activation of neuron $j$ at token $t$ in the expanded hidden layer of the MLP block at layer $l$, and $\mathcal D$ is a small calibration set drawn from the clean training data. As will be discussed in Theorem \ref{thm:surrogate}, this score favors neurons that are not only active, but also influential for changing the adversarial loss.
We rank neurons globally by this score, retain only the top-scoring fraction, and prune the rest when running the inner-loop attack.

Let \(f_{\mathrm{surr}}\) denote the surrogate; we run PGD on it to obtain an adversarial hidden-state perturbation
\begin{equation}
\widetilde{\mathbf H}_{\mathrm{surr}}^\star
=
\argmax_{\widetilde{\mathbf H}\in
\mathcal B(\mathbf H_{l_{\mathrm a}}(x;\theta),\varepsilon;\mathcal T_{\mathrm{LAT}})}
\ell\bigl(f_{\mathrm{surr}}(x;\widetilde{\mathbf H}),y\bigr),
\end{equation}
and transfer the resulting perturbation back to the full defended model. In Section~\ref{sec:exp_surrogate}, we show that this strategy reduces attack-time computation while transferring well to the full defended model.

\section{Theoretical Analysis}
\label{sec:theory}
In this section, we provide theoretical analysis for the suffix-window defense and the surrogate model. We write $\mathbf h_t^{\,l_{\mathrm r}}$ for the hidden vector at token $t$, i.e., the $t$-th row of $\mathbf H_{l_{\mathrm r}}(x;\theta)$. Let $\alpha_{T,t}^l$ and $\hat{\alpha}_{T,t}^l$ denote the attention weights from query position $T$ to token $t$ at layer $l$ in the clean and defended forward passes, respectively, where the clean pass is unperturbed, and the defended pass includes latent perturbation together with ReFT intervention. Here, \(\tilde{\mathbf h}\) denotes the perturbed hidden state before ReFT, and \(\hat{\mathbf h}\) denotes the defended hidden state after ReFT.
We also define the clean and defended value vectors by
\(
\mathbf v_t^l=\mathbf W_V^l \mathbf h_t^{\,l-1}
\)
and
\(
\hat{\mathbf v}_t^l=\mathbf W_V^l \hat{\mathbf h}_t^{\,l-1}.
\)
Here $\mathbf W_Q^l,\mathbf W_K^l,\mathbf W_V^l$ are the query, key, and value matrices of the attention head at layer $l$, with
\(
\|\mathbf W_Q^l\|\le B_Q,\;
\|\mathbf W_K^l\|\le B_K,\;
\|\mathbf W_V^l\|\le B_V
\) for some finite constants $B_Q, B_K, B_V$ 
\citep{edelman2022inductive,he2025impact}. Unless otherwise specified, $\|\cdot\|$ denotes the $\ell_2$ norm for vectors and the spectral norm for matrices. Additional technical details are deferred to Appendix~\ref{app:math_setup}. Besides, we also impose the following assumption:


\begin{assumption}
[Local Lipschitz continuity of attention]
\label{ass:lipschitz_attention} 
In a neighborhood of the clean trajectory, the attention weights at each layer are locally Lipschitz in the query and keys, with constant $L_\alpha$.\end{assumption}

Assumption~\ref{ass:lipschitz_attention} ensures that small perturbations in the queries and keys induce controlled changes in the attention weights. It is standard for softmax-based attention in a bounded neighborhood of the clean trajectory \citep{kim2021lipschitz}.

\subsection{Suffix-Window Defense}
We first analyze the limitation of defending only position $T$. By investigating the layer immediately following the defense layer, namely $l=l_{\mathrm r}+1$, we study how residual perturbations at $l_{\mathrm r}$ propagate to the attention output $\mathbf{o}_T^l$ at the final position $T$ through attention aggregation. Let \(
\mathcal{T}_{\mathrm{atk}}=\{T-k+1,\dots,T\}
\) denote an attacked suffix index set with $k\ge 1$ such that at layer $l_{\mathrm{r}}$ the perturbed hidden states satisfy
\(
\tilde{\mathbf{h}}_t^{\,l_{\mathrm{r}}}=\mathbf{h}_t^{\,l_{\mathrm{r}}}+\Delta_t, 
\)
for all $t\in\mathcal{T}_{\mathrm{atk}}$, and $\Delta_t=\mathbf{0}$ for $t\notin\mathcal{T}_{\mathrm{atk}}$.

\begin{theorem}[Last-token defense is insufficient]
\label{thm:causal_bypass}
Under Assumption~\ref{ass:lipschitz_attention}, further assume that the ReFT operator \(\Phi_{\mathbf R}\) is applied only at position \(T\) and achieves
\(
\|\hat{\mathbf h}_T^{\,l_{\mathrm r}}-\mathbf h_T^{\,l_{\mathrm r}}\|\le \varepsilon_T
\), where $\varepsilon_T$ measures how close the defended hidden state at $T$ is to the clean hidden state after ReFT intervention.
Since ReFT is applied only at position \(T\), the remaining attacked suffix positions are left uncorrected:
\(
\hat{\mathbf h}_t^{\,l_{\mathrm r}}
=
\mathbf h_t^{\,l_{\mathrm r}}+\Delta_t,
t\in\mathcal T_{\mathrm{atk}}\setminus\{T\}.
\)
Let $\hat{\mathbf o}_T^{l}$ denote the defended attention output at $T$ in layer $l$. Let $\mathcal E_T:=\alpha_{T,T}^{l}\|\mathbf W_V^{l}\|\,\varepsilon_T$ and let
\(
\mathcal E_{\mathrm{attn}}
=
\sum_{t=1}^{T}
\bigl|
\hat{\alpha}_{T,t}^{l}-\alpha_{T,t}^{l}
\bigr|
\cdot
\|\hat{\mathbf v}_t^{l}\|.
\)
Then
\begin{equation}
\label{eq:causal_leakage_bounded_main}
\bigl\|
\hat{\mathbf o}_T^{l}-\mathbf o_T^{l}
\bigr\|
\;\ge\;
\left\|
\sum_{t=T-k+1}^{T-1}
\alpha_{T,t}^{l}\,\mathbf W_V^{l}\Delta_t
\right\|
-
\mathcal E_{\mathrm{attn}}
-
\mathcal E_{T}.
\end{equation}
\end{theorem}

The first term on the right-hand side of Eq.~(\ref{eq:causal_leakage_bounded_main}) captures the contribution of upstream perturbations through the value path. The term \(\mathcal E_T\) reflects imperfect ReFT intervention at position \(T\), while \(\mathcal E_{\mathrm{attn}}\) captures perturbation-induced changes in the attention weights through the query/key path. Theorem~\ref{thm:causal_bypass} shows that defending only position \(T\) does not generally block adversarial information from earlier suffix positions. Beyond this attention-level leakage, a more refined analysis also shows that a ReFT module trained only at \(T\) does not transfer to earlier positions (Appendix~\ref{app:Pf_of_Th1} and~\ref{app:generalize}).

In contrast to the above limitation, we next consider a sequence-shared ReFT operator \(\Phi_{\mathrm{shared}}\) trained jointly over the defended suffix window \(\mathcal T_{\mathrm{LAT}}\).

\begin{theorem}[Sequence-shared defense over a suffix window]
\label{thm:seq_shared_defense}
Under Assumption~\ref{ass:lipschitz_attention}, suppose $\Phi_{\mathrm{shared}}$ satisfies
\(
\|\Phi_{\mathrm{shared}}(\tilde{\mathbf h}_t^{\,l_{\mathrm r}})-\mathbf h_t^{\,l_{\mathrm r}}\|
\le \varepsilon_{\mathrm{LAT}},
\)
for all $t \in \mathcal T_{\mathrm{LAT}}$. Here \(\varepsilon_{\mathrm{LAT}}\) quantifies how accurately the shared ReFT operator restores the clean hidden states across the defended suffix window.
Further assume that the defended value vectors satisfy
\(
\|\hat{\mathbf v}_t^l\|\le M_V
\) for some $M_V$
for any $t\in[T]$.
For any token-level attack supported on a suffix set $\mathcal T_{\mathrm{atk}}$ such that $\mathcal T_{\mathrm{atk}}\subseteq \mathcal T_{\mathrm{LAT}}$, the defended attention output at position $T$ in layer $l$ satisfies
\begin{equation}
\|\hat{\mathbf o}_T^l-\mathbf o_T^l\|
\le
\varepsilon_{\mathrm{LAT}}
\bigl(
B_V + M_VL_\alpha(B_Q+LB_K)
\bigr).
\end{equation}
\end{theorem}

Theorem \ref{thm:seq_shared_defense} shows that sequence-shared ReFT over a suffix window controls perturbation propagation through both the value path and the attention-weight path, providing theoretical support for window-based defense rather than single-token defense. The full proof appears in Appendix~\ref{app:Pf_of_Th2}.

\subsection{Effectiveness of the Surrogate Model}
\label{subsec:surrogate}
We next give an informal theorem explaining why a pruned surrogate can remain effective for inner-loop attack generation. The formal version and proof are presented in Appendix~\ref{app:surrogate}.

\begin{theorem}[Surrogate approximation]
\label{thm:surrogate}
Let \(S\) denote the set of neurons retained by the surrogate. Under local regularity conditions, with probability at least $1-2Kp$ for $p\in (0,1/(2K))$, the PGD optimization gap between the full defended model and the surrogate over \(K\) steps satisfies
\begin{equation}
\sum_{s=1}^{K} G_s
\;\le\;
C_{\mathrm{surr}}\,\eta
\sum_{s=1}^{K} M_s(S)
\;+\;
o(K\eta),
\end{equation}

where \(G_s\) is the difference between the one-step improvement in the full inner objective obtained by the full PGD direction and that obtained by the surrogate PGD direction, \(C_{\mathrm{surr}}\) is a constant depending on the Jacobian bound from the attacked hidden state to the MLP expansion layer, \(\eta\) is the PGD step size, and
$
w_{s,j}
:=
\left|a_{s,j}^{(l_{\mathrm a})}(x)\,r_{s,j}^{(l_{\mathrm a})}\right|^2,
M_s(S)
:=
\left(
\sum_{j\notin S} w_{s,j}
\right)^{1/2},
$
where \(w_{s,j}\) measures the act\(\times\)grad contribution of neuron \(j\) at step \(s\), \(M_s(S)\) is the omitted act\(\times\)grad mass of the neurons pruned by the surrogate at step \(s\), and \(r_{s,j}^{(l_{\mathrm a})}\) denotes the corresponding MLP-path gradient component.
\end{theorem}

Theorem~\ref{thm:surrogate} shows that surrogate quality is controlled by the omitted act\(\times\)grad mass \(M_s(S)\). 
At a fixed pruning ratio \(\rho\), an effective surrogate should therefore retain neurons with the largest \(w_{s,j}\). 
In contrast, if random pruning discards a fraction \(\rho\) of neurons uniformly at random, then for the retained set \(S_{\mathrm r}\),
\(
\mathbb E\!\left[M_s(S_{\mathrm r})^2\right]
=
\rho\sum_j w_{s,j}.
\)
Thus, random pruning removes a constant fraction of the total act\(\times\)grad mass in expectation, whereas ActGrad pruning preferentially keeps the largest-contributing neurons and therefore yields a smaller omitted mass at the same sparsity level.

\section{Experiments}

We evaluate the robustness--efficiency tradeoff of LAT-ReFT and test two key design choices: suffix-window placement and ActGrad-based surrogate pruning.
\begin{table*}[t]
\centering
\caption{Main results on IMDB across three representative models. We report clean accuracy (Acc) and attack success rate (ASR) under RandomToken and GCG. We also report training-efficiency statistics: trainable parameters as a percentage of total model parameters and per-step adversarial-training compute (FLOPs/step). See more results and FLOPs details in Appendix~\ref{app:full_results} and~\ref{app:computation}.}
\label{tab:main_imdb}
\setlength{\tabcolsep}{3.8pt}
\renewcommand{\arraystretch}{1.08}
\scriptsize
\begin{tabular}{llccc|cc}
\\
\toprule
& & \multicolumn{3}{c|}{\textbf{Performance}}
& \multicolumn{2}{c}{\textbf{Efficiency}} \\
\cmidrule(lr){3-5} \cmidrule(lr){6-7}
\textbf{Model} & \textbf{Method}
& \textbf{Acc}$\uparrow$
& \textbf{RandomToken}$\downarrow$
& \textbf{GCG}$\downarrow$
& \textbf{Params (\%)}$\downarrow$
& \textbf{FLOPs / step} ($\times 10^{12}$)$\downarrow$ \\
\midrule
\multirow{5}{*}{Llama-3.1-8B}
& No Defense    & 0.98 & 0.93 & 0.97 & --               & --      \\
& R2D2          & \textbf{1.00} & 0.08 & 0.58 & 100.0            & 3750.4    \\
& LAT           & 0.95 & \textbf{0.00} & \textbf{0.02} & 100.0            & 167.0    \\
& CAT           & 0.98 & 0.01 & 0.08 & 0.5559            & 184.5    \\
& \textbf{Ours} & 0.98 & 0.05 & 0.15
                 & \textbf{0.0066} & \textbf{139.8} \\
\midrule
\multirow{5}{*}{Qwen-2.5-3B}
& No Defense    & \textbf{1.00} & 0.98 & 0.99 & --               & --      \\
& R2D2          & 0.96 & 0.07 & 0.35 & 100.0            & 331.2    \\
& LAT           & 0.95 & 0.06 & \textbf{0.10} & 100.0            & 66.4   \\
& CAT           & 0.97 & \textbf{0.03} & 0.35 & 0.9608            & 75.9     \\
& \textbf{Ours} & 0.98 & 0.38 & 0.26
                 & \textbf{0.0085} & \textbf{39.2} \\
\midrule
\multirow{5}{*}{Pythia-1.4B}
& No Defense    & 0.97 & 0.73 & 0.97 & --               & --      \\
& R2D2          & \textbf{1.00} & 0.05 & 0.19 & 100.0            & 161.7     \\
& LAT           & 0.96 & \textbf{0.01} & \textbf{0.14} & 100.0            & 28.6    \\
& CAT           & 0.98 & 0.04 & 0.19 & 0.9505            & 32.2    \\
& \textbf{Ours} & 0.94 & 0.14 & 0.33
                 & \textbf{0.0203} & \textbf{19.6} \\
\bottomrule
\end{tabular}
\end{table*}
\subsection{Experimental Setup}
\paragraph{Datasets and Models.}
Following \citet{howe2024scaling}, we use three classification tasks: IMDB, EnronSpam and PasswordMatch. IMDB and EnronSpam are standard natural-language classification tasks for sentiment analysis and spam detection, respectively, and serve as realistic benchmarks in language understanding settings \citep{maas2011learning, metsis2006spam}. PasswordMatch is a procedurally constructed task inspired by TensorTrust \citep{toyer2023tensor}, and provides a more controlled setting for studying adversarial robustness. We evaluate three pre-trained LLMs: Llama-3.1-8B \citep{grattafiori2024llama}, Qwen-2.5-3B \citep{qwen2.5} and Pythia-1.4B \citep{biderman2023pythia}.
For classification tasks, we attach a task-specific linear classification head to the last-token hidden state and fine-tune each model on clean data to obtain a task-adapted classifier for subsequent adversarial training \citep{howe2024scaling}. Detailed settings are provided in Appendix~\ref{app:datasets_ft}. 

\paragraph{Defenses and attack methods.}
We compare three baselines: R2D2, CAT \citep{xhonneux2024efficient}, and LAT \citep{sheshadri2024latent,casper2024defending}. For R2D2, we follow the version of \citet{howe2024scaling}, which maintains and continually refreshes a pool of previously generated adversarial examples during training rather than generating all attacks from scratch at every step \citep{mazeika2024harmbench}. CAT performs adversarial training in continuous embedding space by applying PGD-based perturbations to input token embeddings, whereas LAT perturbs the model's latent representations in intermediate layers during training. All baselines are implemented in our classification setting; additional details are provided in Appendix~\ref{app:base_defense}.
For our method, we implement LAT-ReFT with the ReFT intervention over a defended suffix window and use the ActGrad-pruned surrogate for inner-loop attack generation. The exact details are provided in Appendix~\ref{app:reft_training}.
We evaluate robustness primarily under token-level suffix attacks, which match the attack setting assumed by our method and theory. Specifically, we consider attacks that append adversarial suffixes: RandomToken \citep{howe2024scaling} and Greedy Coordinate Gradient (GCG) \citep{zou2023universal} as representatives of random sampling and gradient-based optimization. Additional details are provided in Appendix~\ref{app:adv_attacks}. 
The pruned surrogate is used only to generate latent perturbations during training; all evaluation attacks are performed against the full defended model, i.e., the complete backbone together with the trained ReFT intervention. 
    
\subsection{Main Results}
Table~\ref{tab:main_imdb} summarizes the main robustness--efficiency results. 
Full-parameter baselines such as LAT and R2D2 provide strong robustness, but require full-model updates and substantial attack-generation compute; R2D2 may further improve with more training rounds at higher offline cost. CAT reduces the number of trainable parameters while still performing full-model adversarial search. In contrast, our method uses only \(0.0066\%\)--\(0.0203\%\) trainable parameters and reduces per-step FLOPs by \(29.6\%\) on average compared with LAT. Implementation-level measurements further confirm these savings, with 12.7--31.8\% lower step time and 56.6--73.8\% lower peak GPU memory than LAT (Appendix~\ref{app:wallclock}).
As expected, this efficiency gain comes with a loss in absolute robustness compared with full-parameter LAT. For example, on Llama-3.1-8B, LAT achieves the lowest GCG ASR of \(0.02\), while our method obtains \(0.15\) with lower FLOPs and only \(0.0066\%\) trainable parameters. RandomToken is more model-dependent. On Qwen-2.5-3B, our method reduces ASR from \(0.98\) to \(0.38\), but remains less effective than under GCG. One possible reason is that RandomToken repeatedly samples perturbations over the suffix, inducing a broader and less gradient-aligned attack distribution than the perturbations used during training. These random suffixes appear to exploit a negative-to-positive class bias not fully exposed by GCG.
Overall, LAT-ReFT with a circuit-guided surrogate provides a lightweight alternative to full-parameter adversarial training with a favorable robustness--efficiency tradeoff, which is particularly useful when compute budgets make full-parameter adversarial training impractical.

\subsection{Design Analysis of the Defense}
\label{sec:exp_design}
\begin{figure}[t]
    \centering
    \includegraphics[width=\linewidth]{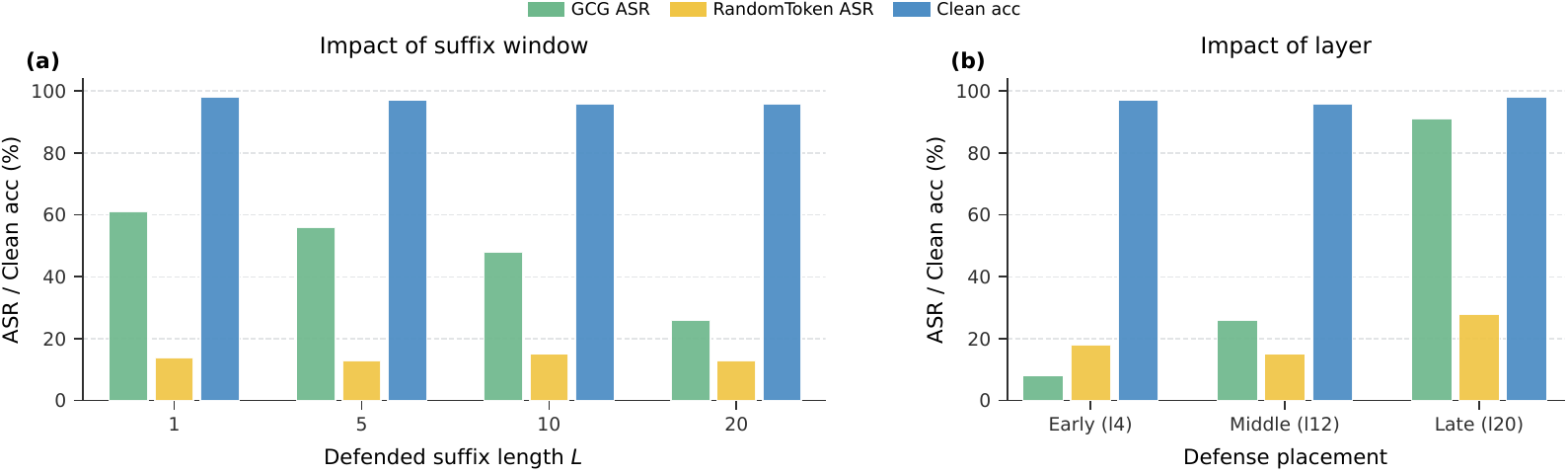}
    \caption{
    Defense-placement ablations on IMDB with Pythia-1.4B under GCG and RandomToken attacks.
    (a) Suffix-window length ablation with the defense layer fixed at \(l=12\).
    (b) Defense-layer ablation with the suffix length fixed at \(L=20\).
    Defending a suffix window and avoiding very late intervention layers both improve robustness, especially under GCG.
    }
    \label{fig:placement}
\end{figure}

We study the two factors of LAT-ReFT mentioned in Section \ref{sec:lat-reft}: the suffix window and the defense layer. Figure~\ref{fig:placement} summarizes the corresponding ablations.

\paragraph{Suffix-window size.}
Figure~\ref{fig:placement}(a) shows that defending only the final suffix token is insufficient, which is consistent with Theorem~\ref{thm:causal_bypass}. Specifically, as the defended suffix length $L$ increases from 1 to 20, the green bars (GCG ASR) decrease substantially, from 61\% to 26\%. This indicates that broader suffix coverage is important for suppressing adversarial leakage from nearby attacked positions, which aligns with Theorem \ref{thm:seq_shared_defense}. By contrast, yellow bars (RandomToken ASR) remain low and vary little with \(L\), likely because its random perturbations are less consistently aligned with the suffix leakage path characterized in Theorem~\ref{thm:causal_bypass}.
The blue bars (clean accuracy) also remain unchanged, showing that the robustness gains under GCG are not driven by sacrificing clean performance. 

\paragraph{Defense layer.}
Figure~\ref{fig:placement}(b) shows that placing the defense too late in the model is ineffective. Specifically, the green bars (GCG ASR) are much higher for late-layer intervention than for early or middle placement, while the yellow bars (RandomToken ASR) show the same qualitative trend at a lower overall level. In contrast, the blue bars (clean accuracy) remain consistently high across all three placements. To explain the worse performance of late layers, the adversarial perturbation has already propagated through many layers by the time it reaches the final layer, leaving less room for a low-rank intervention to remove it reliably. This is broadly consistent with recent work showing that decoder-only Transformers organize computation into distinct depth-wise phases, with later layers acting more as selective refinement stages \citep{queipo2025attention}.

Overall, these results support both parts of our placement strategy: defending a suffix window rather than a single token, and avoiding very late-layer intervention when applying the latent defense.

\subsection{Analysis of Circuit-Guided Surrogate}
\label{sec:exp_surrogate}

We next analyze whether the circuit-guided surrogate can reduce inner-loop attack cost while preserving attack transferability. Figure~\ref{fig:surrogate_analysis} evaluates this design from three aspects: (1) the pruning-ratio tradeoff between training cost and attack success rate, (2) the impact of different neuron-selection rules, and (3) the resulting layer-wise pruning patterns.

\paragraph{Tradeoff between pruning and ASR.}
We first study the tradeoff between the pruning ratio and ASR on IMDB with Pythia-1.4B under GCG attack as a representative setting. 
Figure~\ref{fig:surrogate_analysis}(a) reports the resulting ASR together with the total FLOPs per adversarial training step. The results show that moderate pruning remains effective: at 25\% pruning, the surrogate reduces FLOPs to \(0.804\times\) that of the unpruned model, while ASR increases only slightly from $0.26$ to $0.33$. In contrast, more aggressive pruning sharply degrades transfer: at 50\% pruning, ASR rises to $0.89$, and at 75\% pruning it further increases to $0.95$. We therefore use a pruning ratio of \(25\%\) in all remaining experiments. This degradation is expected: the surrogate is useful only if its PGD directions transfer to the full defended model. Moderate pruning preserves these adversarial search directions, whereas aggressive pruning distorts them and weakens transfer.

\paragraph{Selection rule comparison.}
We next compare three neuron-selection rules at the same pruning ratio of 25\%: random selection (an unstructured pruning baseline), activation-difference scoring (ActDiff), and our activation-times-gradient (ActGrad) rule.
This tests whether surrogate quality depends on which neurons are retained, rather than only on the pruning ratio. ActDiff is motivated by the difference-in-means style of representation analysis widely used in mechanistic interpretability and steering-vector methods \citep{marks2023geometry}. Figure~\ref{fig:surrogate_analysis}(b) shows that the distinction is crucial. While the clean accuracies remain high for all methods, at the same pruning ratio, random pruning gives ASR $0.95$, and ActDiff performs similarly poorly at $0.97$. In contrast, ActGrad reduces ASR to $0.33$. This supports the effectiveness of ActGrad and aligns with Theorem~\ref{thm:surrogate}.

\begin{figure}[t]
    \centering
    \includegraphics[width=1\linewidth]{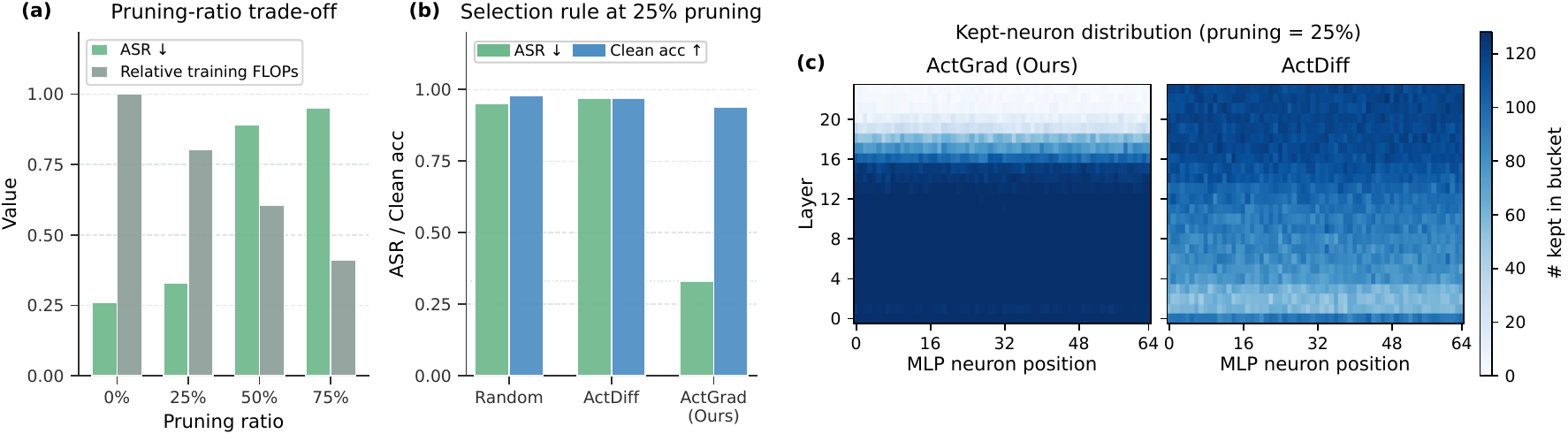}
    \caption{
Analysis of the circuit-guided surrogate on IMDB with Pythia-1.4B under GCG attack.
(a) Tradeoff between ASR and normalized training FLOPs across pruning ratios.
(b) Selection-rule comparison at $25\%$ pruning. ActGrad strongly outperforms random pruning and activation-difference scoring in ASR.
(c) Layer-by-neuron patterns at $25\%$ pruning for ActGrad and ActDiff.
}
    \label{fig:surrogate_analysis}
\end{figure}

\paragraph{Neuron pruning patterns.} 
Figure~\ref{fig:surrogate_analysis}(c) provides a more detailed view of the resulting patterns at the default $25\%$ pruning ratio, using layer-by-neuron heatmaps for ActGrad and ActDiff. The two methods produce different structures: ActGrad preserves more capacity in earlier and middle layers, with a relatively sharp transition toward heavier pruning only in later layers. By contrast, ActDiff allocates its budget more diffusely and places comparatively more mass in later-layer regions. 
Together with the results in Figure~\ref{fig:surrogate_analysis}(b), this suggests that, in our setting, preserving earlier- and middle-layer MLP computation is more important for maintaining transferable adversarial directions. 
This may explain why ActDiff performs slightly worse than random pruning: its retained computation is less concentrated in earlier and middle layers. See more details in Appendix~\ref{app:surrogate_extra}. 

\section{Conclusion}
\label{sec:conclusion}
We proposed an efficient latent adversarial training framework for LLMs that improves the robustness--efficiency tradeoff from both defense and attack sides. On the defense side, LAT-ReFT combines LAT with ReFT and shows that effective intervention should cover a suffix window and avoid very late layers. On the attack side, our circuit-guided surrogate retains high-importance MLP neurons using activation-gradient scores. Together, these components provide a lightweight alternative to full-parameter adversarial training, supported by theoretical analysis and empirical evidence. While our study focuses on LLM-based classifiers, the computational savings of our framework are not tied to the classification head, and extending its robustness evaluation to open-ended generation is a natural next step.

\bibliographystyle{plainnat}
\bibliography{ref}

\newpage
\appendix
\addtocontents{toc}{\protect\setcounter{tocdepth}{2}}
\renewcommand{\contentsname}{Appendix}
\tableofcontents 
\clearpage

\section{Limitations and Future Work}
\label{app:limitations}
In this work, we focus on the adversarial robustness of LLM-based classifiers, following \citet{howe2024scaling}. This setting offers an unambiguous notion of attack success and is of direct practical relevance, since classifiers such as moderation and content filters are routinely deployed as safeguards in security-critical settings \citep{inan2023llama, achara2025watching}. Studying jailbreaks on open-ended tasks, however, requires generative models. Because the inner-loop attack and the outer-loop update that our method accelerates do not depend on the task head, we expect its computational savings to carry over to generative settings, whereas whether the robustness gains transfer requires a dedicated evaluation across generative models and jailbreak attacks, which we leave to future work.

Next, our evaluation follows the suffix-based threat model of \citet{howe2024scaling}, with preliminary results under prefix and infix attacks in Appendix~\ref{app:prefix_attack}. Although our GCG evaluation is white-box and differentiates through the ReFT intervention, it would be valuable to evaluate stronger adaptive attacks that explicitly exploit knowledge of the defense, such as the location of the defended token window or the intervention layer, as well as attacks with larger compute budgets.

Finally, several design choices in our method could be made adaptive. We defend a fixed suffix window and use a fixed MLP pruning ratio for the surrogate; selecting the defended token positions adaptively, or using adaptive pruning and broader circuit extraction, may further improve the robustness--efficiency tradeoff.

\section{Additional Experimental Setup}

\subsection{Datasets and Fine-tuning}
\label{app:datasets_ft}

We use 3 binary classification datasets \citep{howe2024scaling}.
The two natural-language datasets are filtered to examples with 100 to 1{,}000 GPT-2 tokens. The short-text datasets are filtered to examples with 5 to 50 tokens. From the length-filtered training split of each dataset, we sample up to 20{,}000 examples with a fixed random seed ($s = 42$) to form the fine-tuning training set (\texttt{ft\_train}). From the remaining training examples not selected into \texttt{ft\_train}, we further sample up to 100 examples as the attack set, which is disjoint from \texttt{ft\_train} by construction.
For evaluation, we apply the same length filter to the validation split (100–1{,}000 tokens for natural-language datasets; 5–50 tokens for short-text datasets) and retain \emph{all} examples that pass the filter without further subsampling. Table~\ref{tab:dataset_stats} summarizes the resulting split sizes. 
\begin{table*}[ht]
\centering
\caption{Dataset statistics after length filtering.}
\label{tab:dataset_stats}
\small
\begin{tabular}{lrrrrrr}
\toprule
Dataset  & Train (raw) & Train (filtered) & \texttt{ft\_train} & Attack & Val (raw) & Val (filtered) \\
\midrule
IMDB          & 24,365 & 22,446 & 20,000 & 100 & 24,401 & 22,327 \\
EnronSpam     & 29,290 & 20,733 & 20,000 & 100 & 1,852  & 1,289  \\
PasswordMatch & 25,000 & 25,000 & 20,000 & 100 & 25,000 & 25,000 \\
\bottomrule
\end{tabular}
\end{table*}

We fine-tune each model by attaching a two-class linear classification head on top of the final hidden state of the last token and training end-to-end. All models are trained for 3 epochs using AdamW with a learning rate of $10^{-5}$ and linear decay to zero (no warmup). The effective batch size is 16 for all models (per-device batch size 4 with 4 gradient accumulation steps for Llama-3.1-8B).
Input sequences are truncated to 512 tokens for the natural-language datasets (IMDB, Enron) and to 64 tokens for the short-text datasets (PasswordMatch). 8B-scale models are trained in bfloat16. Table~\ref{tab:ft_acc} reports validation accuracy of the resulting classifiers. 

\begin{table*}[ht]
\centering
\caption{Fine-tuned classifier validation accuracy (\%).}
\label{tab:ft_acc}
\setlength{\tabcolsep}{6pt}
\small
\begin{tabular}{lccc}
\toprule
\textbf{Dataset} & \textbf{Pythia-1.4B} & \textbf{Qwen2.5-3B} & \textbf{Llama3.1-8B} \\
\midrule
IMDB          & 94.4 & 96.5 & 97.1 \\
Enron         & 97.8 & 98.4 & 99.2 \\
PasswordMatch & 100.0  & 100.0  & 100.0 \\
\bottomrule
\end{tabular}
\end{table*}

\subsection{Evaluation Metric}
\label{app:evaluation_metric}

We evaluate each defended classifier using clean accuracy and attack success rate (ASR). 
Clean accuracy is computed on the unperturbed test set. For adversarial evaluation, we report ASR over the attack set: the fraction of these examples that are correctly classified before attack and misclassified after attack.

\subsection{LAT-ReFT Training Setup}
\label{app:reft_training}
Algorithm~\ref{alg:main} summarizes one training iteration of LAT-ReFT with the circuit-guided surrogate. For the main IMDB experiments, we train LoReFT interventions on top of task-specific fine-tuned classifiers initialized from the corresponding clean baseline checkpoint. Table~\ref{tab:imdb_reft_hparams} lists the hyperparameters used for the reported IMDB results with ReFT rank=64 for all models.

\begin{table}[ht]
\centering
\caption{LAT-ReFT hyperparameters used for reported IMDB results.}
\label{tab:imdb_reft_hparams}
\footnotesize
\setlength{\tabcolsep}{4pt}
\begin{tabular}{lcccccc}
\toprule
Model & ReFT layer & Attack layer & Pos. & $\epsilon$ & PGD steps & $\lambda_{\mathrm{adv}}$ \\
\midrule
Pythia-1.4B & 12 & 4 & l20 & 0.2 & 8 & 1.0 \\
Llama-3.1-8B & 8 & 4 & l20 & 0.5 & 8 & 1.0 \\
Qwen2.5-3B & 18 & 18 & l20 & 0.5 & 8 & 1.0 \\
\bottomrule
\end{tabular}
\end{table}

\begin{table}[ht]
\centering
\caption{LAT-ReFT hyperparameters for the cross-dataset results on Pythia-1.4B.}
\label{tab:crossdataset_reft_hparams}
\footnotesize
\setlength{\tabcolsep}{4pt}
\begin{tabular}{lccccc}
\toprule
Dataset & ReFT layer & Attack layer & Pos. & $\epsilon$ & $\lambda_{\mathrm{adv}}$ \\
\midrule
EnronSpam & 12 & 12 & l20 & 0.2 & 1.0 \\
PasswordMatch & 18 & 18 & l20 & 0.3 & 1.0 \\
\bottomrule
\end{tabular}
\end{table}

\begin{algorithm}[ht]
\caption{One iteration of LAT-ReFT with circuit-guided surrogate}
\label{alg:main}
\begin{algorithmic}[1]
\REQUIRE Pretrained model backbone \(f_\theta\), surrogate \(f_{\mathrm{surr}}\), attack layer \(l_{\mathrm a}\), defense layer \(l_{\mathrm r}\), defended token set \(\mathcal T_{\mathrm{LAT}}\), PGD steps \(K\), perturbation budget \(\varepsilon\), adversarial weight \(\lambda_{\mathrm{adv}}\), ReFT parameters \((\mathbf R,\mathbf W,\mathbf b)\), learning rate \(\eta\)
\STATE Freeze the backbone parameters \(\theta\); optimize only the ReFT parameters \((\mathbf R,\mathbf W,\mathbf b)\) at layer \(l_{\mathrm r}\)
\FOR{each minibatch \((x,y)\)}
    \STATE Compute clean hidden states \(\mathbf H_{l_{\mathrm a}}(x;\theta)\)
    \STATE Compute clean loss
    \[
    \mathcal L_{\mathrm{clean}}
    \leftarrow
    \ell\!\left(f_{\mathbf R,\mathbf W,\mathbf b}(x),y\right)
    \]

    \STATE Construct the surrogate with the current ReFT parameters at layer \(l_{\mathrm r}\)

    \STATE Initialize \(\widetilde{\mathbf H}^{(0)} \in \mathcal B(\mathbf H_{l_{\mathrm a}}(x;\theta),\varepsilon;\mathcal T_{\mathrm{LAT}})\)

    \FOR{\(s=0,\dots,K-1\)}
        \STATE Compute surrogate inner-loss gradient
        \[
        g^{(s)}
        \leftarrow
        \nabla_{\widetilde{\mathbf H}}
        \ell\!\left(f_{\mathrm{surr}}(x;\widetilde{\mathbf H}^{(s)}),y\right)
        \]
        \STATE Take one PGD step and project back to the latent attack set
        \[
        \widetilde{\mathbf H}^{(s+1)}
        \leftarrow
        \Pi_{\mathcal B(\mathbf H_{l_{\mathrm a}}(x;\theta),\varepsilon;\mathcal T_{\mathrm{LAT}})}
        \Bigl(
        \widetilde{\mathbf H}^{(s)}+\alpha\,\mathrm{sign}(g^{(s)})
        \Bigr)
        \]
    \ENDFOR

    \STATE Set \(\widetilde{\mathbf H}^{\star}\leftarrow \widetilde{\mathbf H}^{(K)}\)

    \STATE Compute adversarial loss on the full defended model
    \[
    \mathcal L_{\mathrm{adv}}
    \leftarrow
    \ell\!\left(f_{\mathbf R,\mathbf W,\mathbf b}(x;\widetilde{\mathbf H}^{\star}),y\right)
    \]

    \STATE Form the outer objective
    \[
    \mathcal L_{\mathrm{total}}
    \leftarrow
    \mathcal L_{\mathrm{clean}}
    +
    \lambda_{\mathrm{adv}}\,\mathcal L_{\mathrm{adv}}
    \]

    \STATE Update only the ReFT parameters
    \[
    (\mathbf R,\mathbf W,\mathbf b)
    \leftarrow
    (\mathbf R,\mathbf W,\mathbf b)
    -
    \eta\,
    \nabla_{(\mathbf R,\mathbf W,\mathbf b)}\mathcal L_{\mathrm{total}}
    \]
\ENDFOR
\end{algorithmic}
\end{algorithm}

\subsection{Baseline Defense Configurations}
\label{app:base_defense}

\paragraph{R2D2.}
We follow the R2D2 adversarial training setup of \citet{howe2024scaling} (Algorithm 1), including its dynamic adversarial pool construction and mixed clean/adversarial minibatch sampling strategy.
Each experiment runs for $R$ adversarial training rounds, where $R$ is chosen per model based on the available compute budget; specifically $R\!=\!5$ for Llama-3.1-8B, $R\!=\!4$ for Pythia-1.4B, and $R\!=\!7$ for Qwen-2.5-3B. In each round we attack $n_{\text{adv}}{=}200$ examples drawn from the attack split using GCG with a suffix of $N{=}10$ tokens, $B{=}128$ candidates per iteration, beam width $k{=}256$, and a linearly increasing iteration budget $k(r)=\mathrm{round}\!\left(k_{\text{start}}+\tfrac{r}{R_{\max}}(k_{\text{end}}-k_{\text{start}})\right)$ with $k_{\text{start}}{=}8$, $k_{\text{end}}{=}32$, $R_{\max}{=}8$, yielding $k\in\{11,14,17,20,23,26,29,32\}$ across rounds. Newly found adversarial examples are added to a persistent pool and resampled via exponential rank-weighting ($\lambda{=}0.005$) that jointly prioritises high-loss and recently generated examples. Each round's fine-tuning minibatch contains $n_{\text{aug}}{=}1000$ items ($80\%$ adversarial, $20\%$ clean), trained for $200$ gradient steps with AdamW (lr\,$=2{\times}10^{-5}$, batch size $8$, no weight decay; for Llama-3.1-8B we use lr\,$=5{\times}10^{-6}$ and batch size $4$ to prevent clean-accuracy collapse).
Total FLOPs reported in Table~\ref{tab:main_imdb} are computed as $C_{\text{adv}} = C_{\text{search}} + C_{\text{train}}$ using the Kaplan formula $C = 6ND$ with model parameter counts $N\in\{1.31\text{B},\,3.09\text{B},\,7.50\text{B}\}$ for Pythia-1.4B, Qwen-2.5-3B, and Llama-3.1-8B respectively; training FLOPs account for less than $3\%$ of the total in all cases.

\paragraph{CAT.}
Continuous Adversarial Training (CAT) applies projected gradient descent (PGD) \emph{in the embedding space} at every training step.
Given a mini-batch with token embeddings $\mathbf{E} \in \mathbb{R}^{B \times T \times d}$, the inner maximization runs $k$ PGD steps within an $\ell_2$ ball of radius $\varepsilon$:
\begin{equation*}
  \boldsymbol{\delta}^{*} = \operatorname*{arg\,max}_{\|\boldsymbol{\delta}\|_2 \le \varepsilon}
  \;\mathcal{L}\!\left(f_\theta(\mathbf{E} + \boldsymbol{\delta}),\, y\right),
\end{equation*}
and the outer minimization updates the model via
\begin{equation*}
  \min_\theta \;\mathcal{L}_{\mathrm{clean}} + \lambda\,\mathcal{L}_{\mathrm{adv}}.
\end{equation*}
In our implemented CAT baseline, we use LoRA adapters on top of a 4-bit quantized backbone rather than full-model fine-tuning, following \citet{xhonneux2024efficient}.
For the final runs, we use LoRA rank $r=16$, LoRA scaling $\alpha=32$, $\varepsilon=0.05$, PGD step size $0.005$, $k=10$, adversarial weight $\lambda=0.5$, learning rate $10^{-4}$, and batch size $8$.

\paragraph{LAT.}
Latent Adversarial Training (LAT) performs adversarial training by injecting perturbations into an intermediate hidden representation rather than the input token embeddings. For an input-label pair \((x,y)\), let \(\mathbf{H}_{l_{\mathrm{a}}}(x)\) denote the hidden states at transformer layer \(l_{\mathrm{a}}\). The inner maximization searches for a latent perturbation within an \(\ell_2\) ball:
\begin{equation*}
  \boldsymbol{\delta}^{*}
  =
  \operatorname*{arg\,max}_{\|\boldsymbol{\delta}\|_2 \le \varepsilon}
  \mathcal{L}\!\left(
    f_{\theta}^{(l_{\mathrm{a}})}(\mathbf{H}_{l_{\mathrm{a}}}(x)+\boldsymbol{\delta}),\, y
  \right),
\end{equation*}
where \(f_{\theta}^{(l_{\mathrm{a}})}\) denotes the remainder of the model after injecting the perturbed activation at layer \(l_{\mathrm{a}}\). We solve this inner problem with projected gradient descent (PGD) using random initialization:
\begin{equation*}
  \boldsymbol{\delta}
  \leftarrow
  \Pi_{\|\boldsymbol{\delta}\|_2 \le \varepsilon}
  \left(
    \boldsymbol{\delta}
    +
    \alpha
    \frac{\nabla_{\boldsymbol{\delta}}\mathcal{L}_{\mathrm{adv}}}
    {\|\nabla_{\boldsymbol{\delta}}\mathcal{L}_{\mathrm{adv}}\|_2}
  \right).
\end{equation*}
The outer objective is
\begin{equation*}
  \min_{\theta}\;
  \mathcal{L}_{\mathrm{clean}}
  +
  \lambda\,\mathcal{L}_{\mathrm{adv}}.
\end{equation*}

In our LAT baseline, we fine-tune all model parameters and restrict the latent perturbation to a fixed subset of positions: the final 20 non-padding tokens at the selected transformer layer, matching the suffix-focused attack setting used in our discrete attacks. For the final IMDB runs, we use \(k=8\) PGD steps, \(\alpha=\varepsilon/5\), and \(\lambda=1.0\). The learning rate is \(2\times 10^{-5}\) for Pythia-1.4B and Qwen-2.5-3B, and \(1\times 10^{-5}\) for Llama-3.1-8B. We report the best layer--radius setting from the LAT sweep: layer 4 with \(\varepsilon=1.0\) for Llama-3.1-8B, layer 8 with \(\varepsilon=1.5\) for Qwen-2.5-3B, and layer 4 with \(\varepsilon=0.5\) for Pythia-1.4B.

\subsection{Adversarial Attacks Setup}
\label{app:adv_attacks}

We evaluate robustness using two adversarial attacks under different suffix-based search strategies. All attacks are evaluated on 100 samples drawn from the designated attack split of the dataset. We report clean accuracy (ACC) and Attack Success Rate (ASR).

\paragraph{GCG (Greedy Coordinate Gradient).}
We use token-level \texttt{GCG}~\citep{zou2023universal} in suffix mode: $N{=}10$ adversarial tokens are appended to the original input.
Each round computes the gradient of the cross-entropy loss with respect to the input embeddings, selects the top-$k{=}256$ candidate replacement tokens per attack position, randomly samples $128$ of these candidates for forward-pass evaluation, and applies the single token substitution that maximises the loss.
The attack runs for $T{=}10$ rounds.

\paragraph{RandomToken.}
\texttt{RandomToken} is a gradient-free baseline designed to be comparable to \texttt{GCG} in the number of adversarial tokens while replacing gradient-guided search with uniform random sampling, following Algorithm 2 of \citet{howe2024scaling}.
At each iteration, all $N{=}10$ adversarial suffix tokens are replaced simultaneously by tokens drawn uniformly at random from the non-special vocabulary.
The attack evaluates the resulting prompt with one forward pass and stops immediately upon a successful flip; otherwise it retains the best-loss candidate seen across all iterations.
Given that an iteration of \texttt{RandomToken} is much cheaper than an iteration of \texttt{GCG}, we use $T=500$ iterations.

\paragraph{Threat model scope.}
Our main evaluation focuses on token-level suffix attacks, which are directly aligned with both the adversarial training setting studied in this paper and the perturbation model used in our theoretical analysis. We do not include semantic rewriting attacks such as PAIR \citep{chao2025jailbreaking} in the main evaluation, since they operate under a broader threat model: they use an external attacker model to iteratively rewrite the input itself. In our classification setting, this would allow the attacker to freely reformulate the original example, making the attack no longer comparable to suffix-based perturbations. 

\subsection{Computation Cost}
\label{app:resources}

Experiments were run on GPU clusters using NVIDIA H200, L40S, and RTX A6000 GPUs. Most jobs used a single GPU worker; CPU workers were used for data loading, preprocessing, and job orchestration. The exact GPU type varied across runs due to shared-cluster scheduling. For the main experiments, we report method-level compute using FLOPs in Table~\ref{tab:main_imdb}. Overall, we estimate that the full project required on the order of several hundred GPU-hours, approximately 800 GPU-hours in total. Measured wall-clock time and peak memory are reported in Appendix~\ref{app:wallclock}.

\section{Additional Results}

\subsection{More Results}
\label{app:full_results}

\begin{table*}[ht]
\centering
\caption{
Additional cross-dataset robustness results on Pythia-1.4B.
We report clean accuracy (Acc, higher is better) and attack success rate (ASR, lower is better) under RandomToken and GCG suffix attacks.
}
\label{tab:additional_pythia_results}
\setlength{\tabcolsep}{5pt}
\renewcommand{\arraystretch}{1.08}
\footnotesize
\begin{tabular}{ll ccc}
\toprule
\textbf{Dataset} & \textbf{Method}
& \textbf{Acc}$\uparrow$
& \textbf{RandomToken ASR}$\downarrow$
& \textbf{GCG ASR}$\downarrow$ \\
\midrule

\multirow{4}{*}{EnronSpam}
& No Defense    & 0.99 & 0.55 & 0.96 \\
& R2D2          & \textbf{1.00} & 0.03 & 0.11 \\
& LAT           & \textbf{1.00} & 0.11 & 0.29 \\
& Ours          & 0.99 & \textbf{0.02} & \textbf{0.08} \\
\midrule

\multirow{4}{*}{PasswordMatch}
& No Defense    & 0.92 & 0.92 & 0.92 \\
& R2D2          & \textbf{1.00} & \textbf{0.00} & \textbf{0.04} \\  
& LAT           & \textbf{1.00} & 0.51 & 0.88 \\
& Ours          & \textbf{1.00} & 0.52 & 0.52 \\


\bottomrule
\end{tabular}
\end{table*}
PasswordMatch inputs are short (12--20 tokens), so an $N{=}10$ token suffix is a much larger relative perturbation than on the natural-language datasets; we therefore use a defended window ($L{=}20$) that exceeds the suffix length, consistent with the trend in Figure~\ref{fig:placement}(a). The remaining difficulty is not specific to our method: full-parameter LAT reaches a GCG ASR of $0.88$, compared with $0.52$ for ours, suggesting that latent-space adversarial training in general transfers poorly to discrete suffixes comparable in length to the prompt. R2D2, which trains directly on GCG-generated adversarial text, remains the strongest defense on this task.

\subsection{Additional Surrogate Analysis}

\begin{table*}[ht]
\centering
\caption{Clean accuracy and GCG attack success rate at different pruning ratios of the surrogate MLP neurons on Pythia-1.4B. The surrogate is constructed by pruning neurons ranked by the proposed activation-gradient importance score. \(0\%\) corresponds to the full unpruned model. Lower GCG ASR is better; higher clean accuracy is better.}
\label{tab:surrogate_appendix}
\setlength{\tabcolsep}{6pt}
\renewcommand{\arraystretch}{1.1}
\footnotesize
\begin{tabular}{ll cc}
\\
\toprule
\textbf{Dataset} & \textbf{Pruning ratio}
& \textbf{Clean Acc}$\uparrow$
& \textbf{GCG ASR}$\downarrow$ \\
\midrule
\multirow{3}{*}{IMDB}
  & 0\% (full) & 0.96 & \textbf{0.26} \\
  & 25\%       & 0.94 & 0.33 \\
  & 50\%       & \textbf{0.97} & 0.89 \\
\midrule
\multirow{3}{*}{EnronSpam}
  & 0\% (full) & \textbf{0.99} & 0.26 \\
  & 25\%       & 0.99 & \textbf{0.08} \\
  & 50\%       & 0.99 & 0.76 \\
\bottomrule
\end{tabular}
\end{table*}
Table~\ref{tab:surrogate_appendix} shows that the optimal pruning ratio can vary across datasets. 
Moderate pruning sometimes improves transferability over the full surrogate, as in EnronSpam, suggesting that pruning may remove task-irrelevant or noisy neurons and produce a cleaner adversarial search direction. 
This is consistent with the circuit-sparsity motivation behind our surrogate construction. 
At the same time, aggressive pruning substantially weakens transfer, indicating that too much pruning destroys attack-relevant geometry.

\subsection{Prefix and Infix Attacks}
\label{app:prefix_attack}

The suffix-window defense is motivated by the classifier readout structure rather than by suffix attacks alone. In decoder-only sequence classification, the final prediction is primarily read from the last hidden state. We therefore expect defending a suffix window near this readout position to remain useful even against prefix attacks.

To verify this, we evaluate GCG on IMDB with Pythia-1.4B, placing the same number of adversarial tokens ($N{=}10$) at the prefix, at an infix position located 90\% of the way through the input sequence \citep{howe2024scaling}, or at the suffix, with all other attack settings as in Appendix~\ref{app:adv_attacks}. As shown in Table~\ref{tab:attack_position}, LAT-ReFT reduces the ASR from $0.97$ to $0.26$ under suffix attacks and from $0.58$ to $0.34$ under infix attacks. For prefix attacks, the undefended model already has a relatively low ASR of $0.26$, and our method maintains a comparable ASR of $0.24$.

This pattern is consistent with the readout structure: since the prediction is read from the final hidden state, perturbations at earlier positions reach the readout only through self-attention and are therefore less effective, which yields the prefix $<$ infix $<$ suffix ordering of ASR in the undefended model. For the same reason, a suffix-window intervention near the readout also mitigates part of the downstream effect of prefix and infix perturbations.

\begin{table}[ht]
\centering
\caption{GCG ASR under different adversarial token positions on IMDB with Pythia-1.4B.}
\label{tab:attack_position}
\small
\begin{tabular}{lccc}
\toprule
Method & Suffix ASR$\downarrow$ & Prefix ASR$\downarrow$ & Infix ASR$\downarrow$ \\
\midrule
No Defense & 0.97 & 0.26 & 0.58 \\
LAT-ReFT   & \textbf{0.26} & \textbf{0.24} & \textbf{0.34} \\
\bottomrule
\end{tabular}
\end{table}

\subsection{Effect of ReFT Rank}
\label{app:rank}
We examine the effect of the ReFT rank on IMDB with Llama-3.1-8B under GCG, keeping all other settings fixed (Table~\ref{tab:rank}). Very small ranks provide insufficient capacity to correct adversarially shifted representations and lead to high ASR, whereas the ASR decreases substantially as the rank grows. This is consistent with prior evidence that adversarial training benefits more from increased model capacity than standard training \citep{xie2019intriguing}, while low ranks are often sufficient for standard downstream adaptation \citep{wu2024reft}. We therefore adopt rank 64 as a unified default across models and datasets; further increasing the rank increases the number of trainable parameters and computation, yielding a capacity--efficiency trade-off.

\begin{table}[ht]
\centering
\caption{Effect of ReFT rank on IMDB with Llama-3.1-8B under GCG.}
\label{tab:rank}
\small
\begin{tabular}{cccc}
\toprule
Rank & Params (\%)$\downarrow$ & Clean Acc$\uparrow$ & GCG ASR$\downarrow$ \\
\midrule
4  & 0.0004 & 0.98 & 0.94 \\
16 & 0.0017 & 0.97 & 0.87 \\
32 & 0.0033 & 0.98 & 0.56 \\
64 & 0.0066 & 0.98 & \textbf{0.15} \\
\bottomrule
\end{tabular}
\end{table}

\section{Circuit-Guided Surrogate Details}
\label{app:surrogate_extra}
\subsection{ActDiff Method}
ActDiff is motivated by the difference-in-means style of representation analysis widely used in mechanistic interpretability and steering-vector methods. Concretely, it ranks neurons by the absolute class-conditional gap in mean activation magnitude:
\begin{equation}
s^{\mathrm{ActDiff}}_{l,j}
=
\left|
\frac{1}{|\mathcal{D}_{+}|}
\sum_{x \in \mathcal{D}_{+}} \sum_{t} \left|a^{(l)}_{t,j}(x)\right|
-
\frac{1}{|\mathcal{D}_{-}|}
\sum_{x \in \mathcal{D}_{-}} \sum_{t} \left|a^{(l)}_{t,j}(x)\right|
\right|,
\end{equation}
where $\mathcal{D}_{+}, \mathcal{D}_{-}$ are the positive and negative calibration subsets. Unlike ActGrad, however, this rule depends only on clean activation statistics and does not measure how much a neuron actually affects the attack objective.

\subsection{Surrogate pruning heatmaps}
We further visualize the layerwise pruning patterns for two representative datasets. 
These heatmaps illustrate how the pruned MLP neurons are distributed across layers for different model families.

\begin{figure*}[ht]
\centering 
\includegraphics[width=\textwidth]{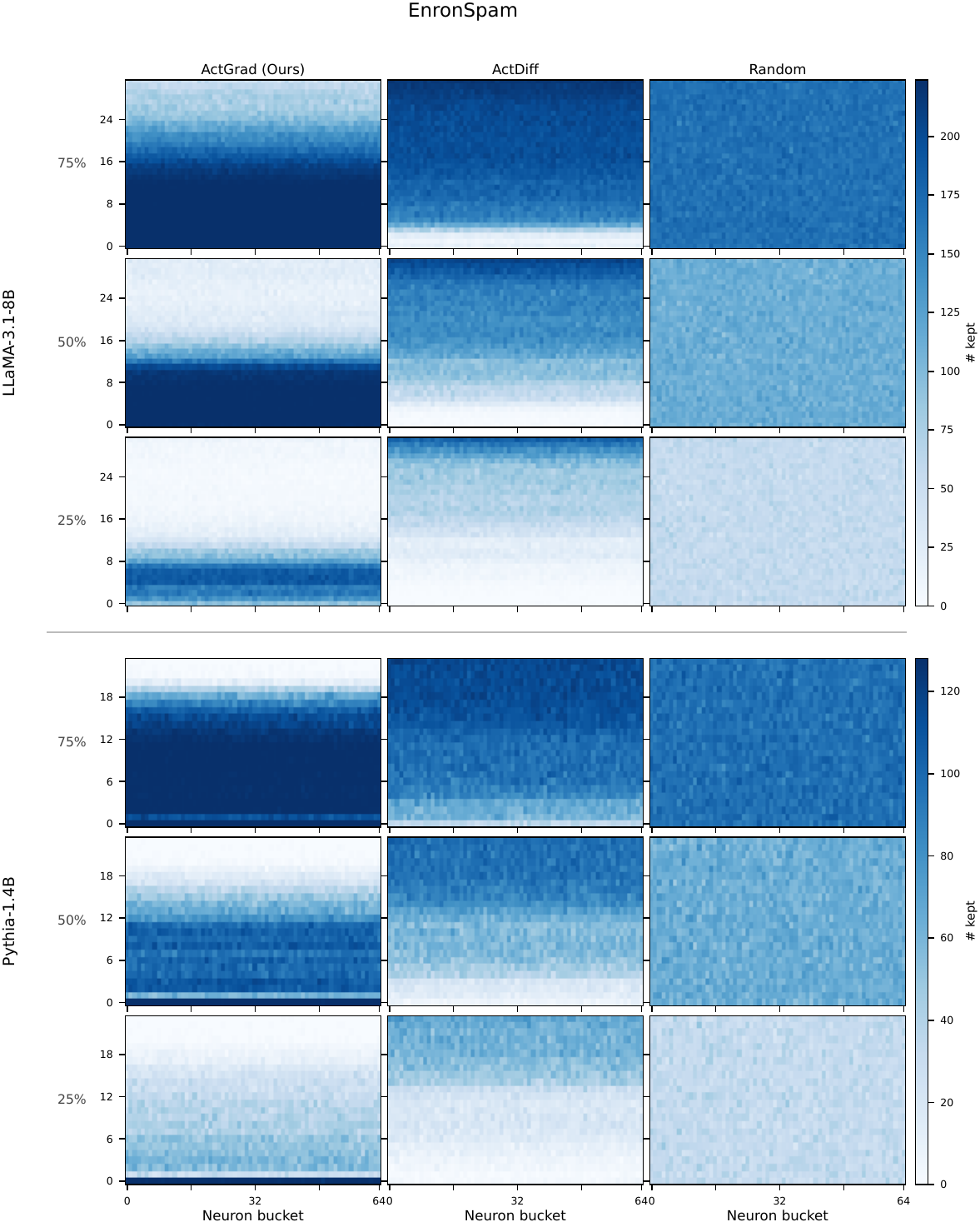} 
\caption{Heatmaps of layerwise surrogate pruning patterns on EnronSpam. 75\% denotes retaining 75\% of MLP neurons (i.e., pruning 25\%).} 
\label{fig:surrogate_heatmaps} 
\end{figure*} 

\begin{figure*}[ht] 
\centering 
\includegraphics[width=\textwidth]{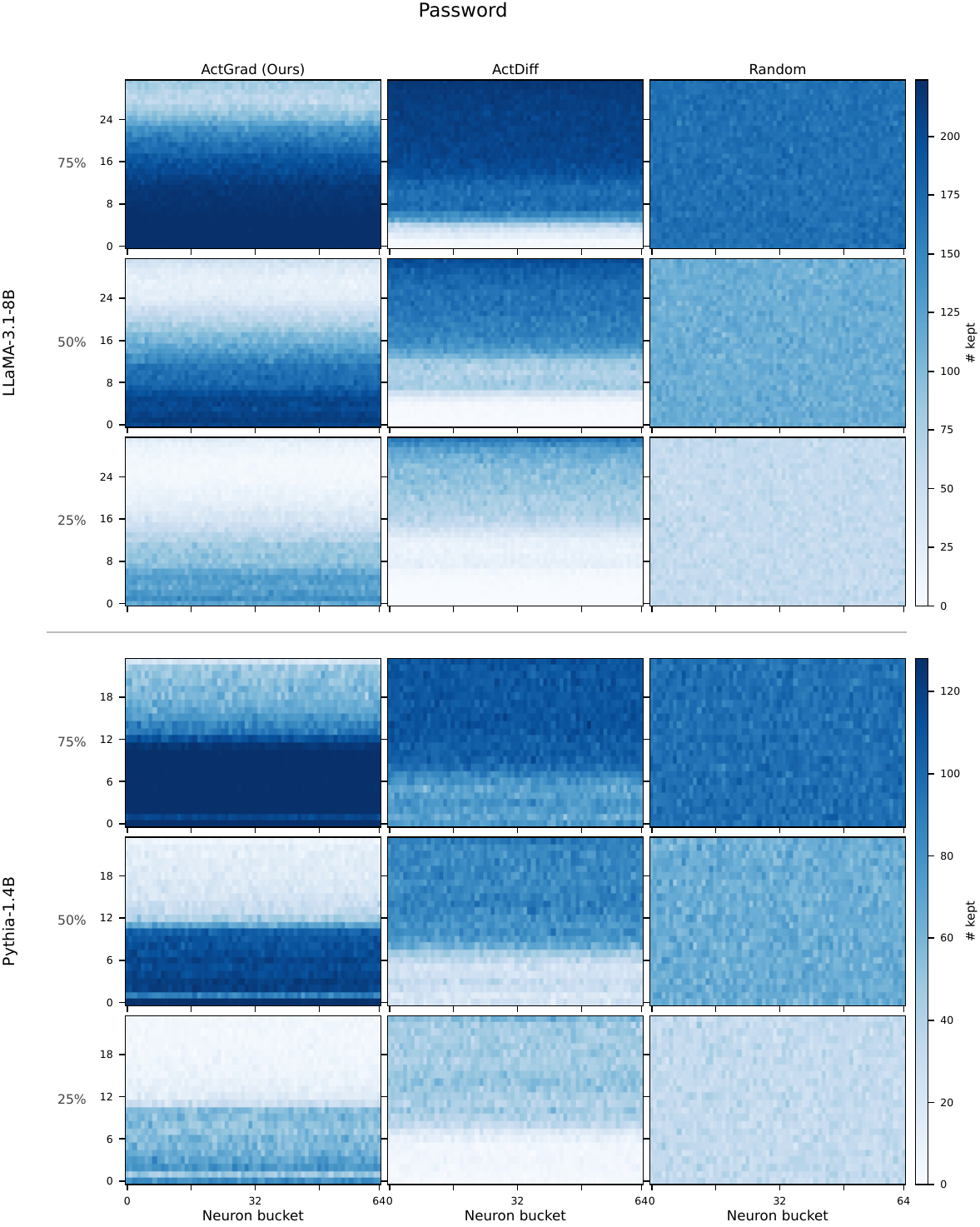} 
\caption{Heatmaps of layerwise surrogate pruning patterns on PasswordMatch.} 
\label{fig:surrogate_heatmaps_password} 
\end{figure*}

\clearpage
\section{Additional Theory and Proofs}
\subsection{Complete Mathematical Setup}
\label{app:math_setup}

We provide the full attention-level setup used in the theoretical analysis. As in the main text, we analyze a single attention head at a fixed layer $l=l_{\mathrm{r}}+1$, since the multi-head case only changes constants. For the final token at position $T$, define
\[
\mathbf q_T^l = \mathbf W_Q^l \mathbf h_T^{\,l_{\mathrm{r}}},\quad
\mathbf k_t^l = \mathbf W_K^l \mathbf h_t^{\,l_{\mathrm{r}}},\quad
\mathbf v_t^l = \mathbf W_V^l \mathbf h_t^{\,l_{\mathrm{r}}},
\]
and let the corresponding attention output be
\[
\mathbf o_T^l
=
\sum_{t=1}^T \alpha_{T,t}^l \mathbf v_t^l,
\quad
\alpha_{T,t}^l
=
\frac{
\exp\!\left((\mathbf q_T^l)^\top \mathbf k_t^l/\sqrt d\right)
}{
\sum_{j=1}^T
\exp\!\left((\mathbf q_T^l)^\top \mathbf k_j^l/\sqrt d\right)
}.
\]

For completeness, we also restate the ReFT operator used throughout the proofs. Let $\mathbf R\in\mathbb R^{r\times d}$ be row-orthonormal, i.e.,
\[
\mathbf R\mathbf R^\top=\mathbf I_r,
\]
and let $\mathbf W\in\mathbb R^{r\times d}$ and $\mathbf b\in\mathbb R^r$. For a hidden state $\mathbf h\in\mathbb R^d$, the ReFT intervention is
\[
\Phi_{\mathbf R}(\mathbf h)
=
\mathbf h+\mathbf R^\top(\mathbf W\mathbf h+\mathbf b-\mathbf R\mathbf h)
=
(\mathbf I-\mathbf R^\top\mathbf R)\mathbf h
+
\mathbf R^\top(\mathbf W\mathbf h+\mathbf b).
\]
Since $\mathbf R$ has orthonormal rows, $\mathbf R^\top\mathbf R$ is the orthogonal projector onto the edited subspace $\mathrm{span}(\mathbf R)$, and
\[
\Pi_{\mathrm{span}(\mathbf R)^\perp}
:=
\mathbf I-\mathbf R^\top\mathbf R
\]
is the orthogonal projector onto its orthogonal complement.

\subsection{Proof of Theorem~\ref{thm:causal_bypass}}
\label{app:Pf_of_Th1}
\begin{proof}
We analyze the attention layer $l=l_{\mathrm{r}}+1$ immediately after the defense layer. Let $(\mathbf q_T^l,\mathbf k_t^l,\mathbf v_t^l,\mathbf o_T^l)$ and $(\hat{\mathbf q}_T^l,\hat{\mathbf k}_t^l,\hat{\mathbf v}_t^l,\hat{\mathbf o}_T^l)$ denote the clean and defended query, key, value, and attention output, respectively.

For attacked suffix positions $t\in\mathcal T_{\mathrm{atk}}\setminus\{T\}$, no correction is applied, so
\[
\hat{\mathbf h}_t^{\,l_{\mathrm{r}}}
=
\mathbf h_t^{\,l_{\mathrm{r}}}+\Delta_t,
\]
which implies
\[
\hat{\mathbf k}_t^l=\mathbf k_t^l+\mathbf W_K^l\Delta_t,
\quad
\hat{\mathbf v}_t^l=\mathbf v_t^l+\mathbf W_V^l\Delta_t.
\]
At the defended token $T$, the ReFT correction gives
\[
\bigl\|
\hat{\mathbf h}_T^{\,l_{\mathrm{r}}}
-
\mathbf h_T^{\,l_{\mathrm{r}}}
\bigr\|
\le \varepsilon_T,
\]
and therefore
\[
\|\hat{\mathbf q}_T^l-\mathbf q_T^l\|\le \|\mathbf W_Q^l\|\varepsilon_T,
\quad
\|\hat{\mathbf k}_T^l-\mathbf k_T^l\|\le \|\mathbf W_K^l\|\varepsilon_T,
\quad
\|\hat{\mathbf v}_T^l-\mathbf v_T^l\|\le \|\mathbf W_V^l\|\varepsilon_T.
\]

We decompose the output difference as
\begin{align*}
\hat{\mathbf o}_T^l-\mathbf o_T^l
&=
\sum_{t=1}^T \alpha_{T,t}^l(\hat{\mathbf v}_t^l-\mathbf v_t^l)
+
\sum_{t=1}^T(\hat{\alpha}_{T,t}^l-\alpha_{T,t}^l)\hat{\mathbf v}_t^l \notag\\
&=
\sum_{t=T-k+1}^{T-1}\alpha_{T,t}^l\mathbf W_V^l\Delta_t
+
\alpha_{T,T}^l(\hat{\mathbf v}_T^l-\mathbf v_T^l)
+
\sum_{t=1}^T(\hat{\alpha}_{T,t}^l-\alpha_{T,t}^l)\hat{\mathbf v}_t^l.
\end{align*}
Applying the reverse triangle inequality gives
\begin{align*}
\bigl\|
\hat{\mathbf o}_T^{l}-\mathbf o_T^{l}
\bigr\|
&\ge
\left\|
\sum_{t=T-k+1}^{T-1}\alpha_{T,t}^{l}\mathbf W_V^{l}\Delta_t
\right\|
-
\alpha_{T,T}^{l}\|\hat{\mathbf v}_T^{l}-\mathbf v_T^{l}\|
-
\left\|
\sum_{t=1}^{T}
(\hat{\alpha}_{T,t}^{l}-\alpha_{T,t}^{l})\hat{\mathbf v}_t^{l}
\right\|
\notag\\
&\ge
\left\|
\sum_{t=T-k+1}^{T-1}\alpha_{T,t}^{l}\mathbf W_V^{l}\Delta_t
\right\|
-
\alpha_{T,T}^{l}\|\mathbf W_V^{l}\|\,\varepsilon_T
-
\sum_{t=1}^{T}
\bigl|
\hat{\alpha}_{T,t}^{l}-\alpha_{T,t}^{l}
\bigr|
\cdot
\|\hat{\mathbf v}_t^{l}\|.
\end{align*}
This proves Eq.~(\ref{eq:causal_leakage_bounded_main}) with
\[
\mathcal E_T=\alpha_{T,T}^{l}\|\mathbf W_V^{l}\|\,\varepsilon_T, \quad
\mathcal E_{\mathrm{attn}}
=
\sum_{t=1}^{T}
\bigl|
\hat{\alpha}_{T,t}^{l}-\alpha_{T,t}^{l}
\bigr|
\cdot
\|\hat{\mathbf v}_t^{l}\|.
\]

It remains to bound $\mathcal E_{\mathrm{attn}}$. 
The attention weights depend on the query $\mathbf q_T^l$ and the key collection
$\{\mathbf k_t^l\}_{t=1}^T$ through the softmax map. Under Assumption~\ref{ass:lipschitz_attention}, there exists a local Lipschitz constant $L_\alpha>0$ such that
\[
\sum_{t=1}^{T}
\bigl|
\hat{\alpha}_{T,t}^{l}-\alpha_{T,t}^{l}
\bigr|
\le
L_\alpha
\left(
\|\hat{\mathbf q}_T^{l}-\mathbf q_T^{l}\|
+
\sum_{t=1}^{T}\|\hat{\mathbf k}_t^{l}-\mathbf k_t^{l}\|
\right).
\]
Using the bounds established above,
\[
\|\hat{\mathbf q}_T^{l}-\mathbf q_T^{l}\|
\le
\|\mathbf W_Q^{l}\|\,\varepsilon_T,
\]
and
\[
\sum_{t=1}^{T}\|\hat{\mathbf k}_t^{l}-\mathbf k_t^{l}\|
\le
\|\mathbf W_K^{l}\|\,\varepsilon_T
+
\|\mathbf W_K^{l}\|
\sum_{t=T-k+1}^{T-1}\|\Delta_t\|.
\]
Therefore,
\[
\sum_{t=1}^{T}
\bigl|
\hat{\alpha}_{T,t}^{l}-\alpha_{T,t}^{l}
\bigr|
\le
L_\alpha
\left(
\|\mathbf W_Q^{l}\|\,\varepsilon_T
+
\|\mathbf W_K^{l}\|\,\varepsilon_T
+
\|\mathbf W_K^{l}\|
\sum_{t=T-k+1}^{T-1}\|\Delta_t\|
\right).
\]

If the defended trajectory remains in a bounded neighborhood of the clean one, then there exists a constant $M_V>0$ such that
\[
\sup_{1\le t\le T}\|\hat{\mathbf v}_t^{l}\|\le M_V.
\]
Hence
\begin{align*}
\mathcal E_{\mathrm{attn}}
&\le
M_V
\sum_{t=1}^{T}
\bigl|
\hat{\alpha}_{T,t}^{l}-\alpha_{T,t}^{l}
\bigr|
\notag\\
&\le
M_V L_\alpha
\left(
(\|\mathbf W_Q^{l}\|+\|\mathbf W_K^{l}\|)\varepsilon_T
+
\|\mathbf W_K^{l}\|
\sum_{t=T-k+1}^{T-1}\|\Delta_t\|
\right).
\end{align*}
This bound on $\mathcal E_{\mathrm{attn}}$ is used in Corollary~\ref{cor:leakage} below to identify when the lower bound in Theorem~\ref{thm:causal_bypass} is non-trivial.
\end{proof}

\begin{corollary}[Non-trivial leakage under last-token defense]
\label{cor:leakage}
In the setting of Theorem~\ref{thm:causal_bypass}, further assume $\|\hat{\mathbf v}_t^{l}\|\le M_V$ for all $t\in[T]$, and that the uncorrected upstream perturbations share a common direction, i.e., $\Delta_t=\delta_t\mathbf u$ with $\|\mathbf u\|=1$ and $\delta_t\ge 0$ for $t\in\mathcal T_{\mathrm{atk}}\setminus\{T\}$.
Define
\[
\gamma:=\|\mathbf W_V^{l}\mathbf u\|\min_{t\in\mathcal T_{\mathrm{atk}}\setminus\{T\}}\alpha_{T,t}^{l},
\qquad
\kappa:=M_V L_\alpha\|\mathbf W_K^{l}\|.
\]
Then
\[
\bigl\|\hat{\mathbf o}_T^{l}-\mathbf o_T^{l}\bigr\|
\;\ge\;
(\gamma-\kappa)\sum_{t=T-k+1}^{T-1}\delta_t
\;-\;
\Bigl(\alpha_{T,T}^{l}\|\mathbf W_V^{l}\|
+M_V L_\alpha\bigl(\|\mathbf W_Q^{l}\|+\|\mathbf W_K^{l}\|\bigr)\Bigr)\varepsilon_T .
\]
In particular, if $\gamma>\kappa$, the deviation of the defended attention output grows linearly with the magnitude of the upstream perturbations, even when the last token is restored exactly ($\varepsilon_T=0$).
\end{corollary}

\begin{proof}
Since $\alpha_{T,t}^{l}\ge 0$ and $\delta_t\ge 0$,
\[
\Bigl\|\sum_{t=T-k+1}^{T-1}\alpha_{T,t}^{l}\mathbf W_V^{l}\Delta_t\Bigr\|
=\|\mathbf W_V^{l}\mathbf u\|\sum_{t=T-k+1}^{T-1}\alpha_{T,t}^{l}\delta_t
\ge\gamma\sum_{t=T-k+1}^{T-1}\delta_t .
\]
Combining this with $\mathcal E_T=\alpha_{T,T}^{l}\|\mathbf W_V^{l}\|\varepsilon_T$ and the bound on $\mathcal E_{\mathrm{attn}}$ derived above (with $\|\Delta_t\|=\delta_t$) in Eq.~(\ref{eq:causal_leakage_bounded_main}) yields the claim.

\end{proof}

\subsection{Proof of Theorem~\ref{thm:seq_shared_defense}}
\label{app:Pf_of_Th2}
\begin{proof}
Because $k\le L$, the attack window is contained in the defended suffix window, i.e.,
\[
\mathcal T_{\mathrm{atk}}\subseteq \mathcal T_{\mathrm{LAT}}.
\]
Hence, for every attacked position $t\in\mathcal T_{\mathrm{atk}}$, the sequence-shared operator satisfies
\[
\|\Phi_{\mathrm{shared}}(\tilde{\mathbf h}_t^{\,l_{\mathrm{r}}})-\mathbf h_t^{\,l_{\mathrm{r}}}\|
\le
\varepsilon_{\mathrm{LAT}}.
\]
Therefore, the corresponding query, key, and value perturbations are bounded by
\[
\|\hat{\mathbf q}_T^l-\mathbf q_T^l\|\le B_Q\varepsilon_{\mathrm{LAT}},
\quad
\|\hat{\mathbf k}_t^l-\mathbf k_t^l\|\le B_K\varepsilon_{\mathrm{LAT}},
\quad
\|\hat{\mathbf v}_t^l-\mathbf v_t^l\|\le B_V\varepsilon_{\mathrm{LAT}}.
\]

We decompose the attention distortion into a value term and an attention-weight term:
\[
\|\hat{\mathbf o}_T^l-\mathbf o_T^l\|
\le
\sum_{t=1}^T \alpha_{T,t}^l \|\hat{\mathbf v}_t^l-\mathbf v_t^l\|
+
\sum_{t=1}^T |\hat{\alpha}_{T,t}^l-\alpha_{T,t}^l|\,\|\hat{\mathbf v}_t^l\|.
\]

For the value term, only positions in $\mathcal T_{\mathrm{LAT}}$ contribute, and \(\sum_t \alpha_{T,t}^l=1\), so
\[
\sum_{t=1}^T \alpha_{T,t}^l \|\hat{\mathbf v}_t^l-\mathbf v_t^l\|
\le
B_V\varepsilon_{\mathrm{LAT}}.
\]

For the attention-weight term, Assumption~\ref{ass:lipschitz_attention} gives
\[
\sum_{t=1}^T |\hat{\alpha}_{T,t}^l-\alpha_{T,t}^l|
\le
L_\alpha
\left(
\|\hat{\mathbf q}_T^l-\mathbf q_T^l\|
+
\sum_{t=1}^T \|\hat{\mathbf k}_t^l-\mathbf k_t^l\|
\right).
\]
Positions outside $\mathcal T_{\mathrm{LAT}}$ are neither perturbed nor intervened on,
so $\hat{\mathbf k}_t^l=\mathbf k_t^l$ there. For $t\in\mathcal T_{\mathrm{LAT}}$, the
ReFT operator is applied, and the assumption on $\Phi_{\mathrm{shared}}$ gives
$\|\hat{\mathbf h}_t^{\,l_{\mathrm r}}-\mathbf h_t^{\,l_{\mathrm r}}\|\le\varepsilon_{\mathrm{LAT}}$
(for unattacked $t\in\mathcal T_{\mathrm{LAT}}$, $\tilde{\mathbf h}_t^{\,l_{\mathrm r}}=\mathbf h_t^{\,l_{\mathrm r}}$).
Hence $\|\hat{\mathbf k}_t^l-\mathbf k_t^l\|\le B_K\varepsilon_{\mathrm{LAT}}$, and since
$|\mathcal T_{\mathrm{LAT}}|=L$,
\[
\sum_{t=1}^T \|\hat{\mathbf k}_t^l-\mathbf k_t^l\|
\le
L B_K \varepsilon_{\mathrm{LAT}}.
\]
Therefore,
\[
\sum_{t=1}^T |\hat{\alpha}_{T,t}^l-\alpha_{T,t}^l|
\le
L_\alpha\varepsilon_{\mathrm{LAT}}(B_Q+LB_K).
\]
Assuming \(\|\hat{\mathbf v}_t^l\|\le M_V\), we obtain
\[
\sum_{t=1}^T |\hat{\alpha}_{T,t}^l-\alpha_{T,t}^l|\,\|\hat{\mathbf v}_t^l\|
\le
M_VL_\alpha\varepsilon_{\mathrm{LAT}}(B_Q+LB_K).
\]

Combining the two bounds yields
\[
\|\hat{\mathbf o}_T^l-\mathbf o_T^l\|
\le
\varepsilon_{\mathrm{LAT}}
\bigl(
B_V + M_VL_\alpha(B_Q+LB_K)
\bigr),
\]
which proves the theorem.
\end{proof}

\subsection{Proof of Theorem~\ref{thm:surrogate}}
\label{app:surrogate}

\begin{lemma}
\label{lem:jacobian_z}
To simplify notation, write the MLP expansion layer at \(l_{\mathrm a}\) as
\(
Z=\phi(H^{\mathrm{attn}}W_1),
\)
where
\[
H^{\mathrm{attn}}
=
\sigma\!\left(
\operatorname{RowSoftmax}\!\left(
\frac{\widetilde{\mathbf H}W_{QK}\widetilde{\mathbf H}^\top}{\sqrt{d_m}}
\right)\widetilde{\mathbf H}W_V
\right),
\quad
W_{QK}=W_QW_K^\top,
\]
and \(\sigma\) and \(\phi\) are \(L_\sigma\)- and \(L_\phi\)-Lipschitz, respectively. Assume that 
\[
\|W_Q\|\le B_Q,\quad \|W_K\|\le B_K,\quad \|W_V\|\le B_V.
\]
Then there exists \(C_{\mathrm{attn}}=O(T)\) such that
\[
\left\|\frac{\partial Z}{\partial \widetilde{\mathbf H}}\right\|
\le
L_\phi L_\sigma \|W_1\|
\left(
C_{\mathrm{attn}}\frac{B_QB_KB_V}{\sqrt{d_m}}
+
B_V
\right).
\]
\end{lemma}

\begin{proof}[Proof of Lemma~\ref{lem:jacobian_z}]
By the chain rule,
\[
\frac{\partial Z}{\partial \widetilde{\mathbf H}}
=
\frac{\partial Z}{\partial H^{\mathrm{attn}}}
\frac{\partial H^{\mathrm{attn}}}{\partial \widetilde{\mathbf H}}.
\]
Since \(\phi\) is \(L_\phi\)-Lipschitz,
\[
\left\|\frac{\partial Z}{\partial H^{\mathrm{attn}}}\right\|
\le
L_\phi \|W_1\|.
\]

Let
\(
A(\widetilde{\mathbf H})
=
\operatorname{RowSoftmax}\!\left(
\frac{\widetilde{\mathbf H}W_{QK}\widetilde{\mathbf H}^\top}{\sqrt{d_m}}
\right),
\)
so that
\(
H^{\mathrm{attn}}
=
\sigma\!\left(
A(\widetilde{\mathbf H})\,\widetilde{\mathbf H}W_V
\right).
\)
Since \(\sigma\) is \(L_\sigma\)-Lipschitz,
\[
\left\|\frac{\partial H^{\mathrm{attn}}}{\partial \widetilde{\mathbf H}}\right\|
\le
L_\sigma
\left\|
\frac{\partial \bigl(A(\widetilde{\mathbf H})\widetilde{\mathbf H}W_V\bigr)}
{\partial \widetilde{\mathbf H}}
\right\|.
\]

For a single-row softmax \(p=\mathrm{softmax}(z)\), its Jacobian is
\[
D\,\mathrm{softmax}(z)=\operatorname{diag}(p)-pp^\top,
\]
and satisfies the uniform spectral-norm bound \(\|D\,\mathrm{softmax}(z)\|\le \tfrac12\) \citep{nair2025softmax}. Therefore the RowSoftmax map is uniformly Lipschitz row-wise. Then the derivative of \(A(\widetilde{\mathbf H})\,\widetilde{\mathbf H}W_V\) with respect to \(\widetilde{\mathbf H}\) consists of a value-path term and an attention-weight term. The value-path term is bounded by \(B_V\). For the attention-weight term, the softmax Jacobian bound together with
\[
\|W_{QK}\|
\le
\|W_Q\|\,\|W_K\|
\le
B_QB_K,
\quad
\|W_V\|\le B_V,
\]
gives a contribution proportional to
\(
\frac{\|W_{QK}\|\|W_V\|}{\sqrt{d_m}},
\)
multiplied by a sequence-length accumulation factor. Absorbing this factor into \(C_{\mathrm{attn}}=O(T)\), we obtain
\[
\left\|
\frac{\partial H^{\mathrm{attn}}}{\partial \widetilde{\mathbf H}}
\right\|
\le
L_\sigma
\left(
C_{\mathrm{attn}}\frac{B_QB_KB_V}{\sqrt{d_m}}
+
B_V
\right).
\]

Combining the above displays gives
\[
\left\|\frac{\partial Z}{\partial \widetilde{\mathbf H}}\right\|
\le
L_\phi L_\sigma \|W_1\|
\left(
C_{\mathrm{attn}}\frac{B_QB_KB_V}{\sqrt{d_m}}
+
B_V
\right).
\]
\end{proof}

\begin{assumption}[Non-degeneracy on pruned active coordinates]
\label{ass:nondeg_pruned}
There exist \(m_0>0\) and \(p\in(0,1)\) such that, with probability at least \(1-p\),
\[
\min_{j\notin S}\left|a_{t,j}^{(l_{\mathrm a})}(x)\right|
\ge
m_0,
\]
for all \(t\in[K]\).
\end{assumption}

\begin{assumption}[Local sign-stability]
\label{ass:sign_stability}
There exists a local constant \(c>0\) such that, on the same event,
\[
\langle g_t,u_t\rangle-\langle g_t,u_t^s\rangle
\le
c\,\|g_t-g_t^s\|,
\]
for all \(t\in[K]\), where
\[
g_t=\nabla_{\widetilde{\mathbf H}}L_f\!\left(\widetilde{\mathbf H}^{(t)}\right),
\quad
g_t^s=\nabla_{\widetilde{\mathbf H}}L_s\!\left(\widetilde{\mathbf H}^{(t)}\right),
\]
and
\(
u_t:=\operatorname{sign}(g_t),
u_t^s:=\operatorname{sign}(g_t^s)
\)
be the one-step \(\ell_\infty\)-PGD ascent directions.
\end{assumption}

\begin{proof}[Proof of Theorem~\ref{thm:surrogate}]
We work on the event on which Assumptions~\ref{ass:nondeg_pruned} and~\ref{ass:sign_stability} hold.
Define
\[
G_t
:=
L_f\!\left(\widetilde{\mathbf H}^{(t)}+\eta u_t\right)
-
L_f\!\left(\widetilde{\mathbf H}^{(t)}+\eta u_t^s\right).
\]

Let \(Z\) denote the MLP expansion layer at \(l_{\mathrm a}\), and let
\(
r_t:=\nabla_Z L_f\!\left(\widetilde{\mathbf H}^{(t)}\right).
\)
Write \(P_S\) for the projection onto the retained neuron set \(S\). Define
\[
s_{t,j}
:=
\left|a_{t,j}^{(l_{\mathrm a})}(x)\,r_{t,j}^{(l_{\mathrm a})}\right|,
\quad
M_t
:=
\left(
\sum_{j\notin S}
|a_{t,j}^{(l_{\mathrm a})}(x)\,r_{t,j}^{(l_{\mathrm a})}|^2
\right)^{1/2}.
\]

By the first-order expansion of \(L_f\),
\[
L_f(\widetilde{\mathbf H}^{(t)}+\eta u)
=
L_f(\widetilde{\mathbf H}^{(t)})
+
\eta\langle g_t,u\rangle
+
o(\eta),
\]
hence
\[
G_t
=
\eta\bigl(\langle g_t,u_t\rangle-\langle g_t,u_t^s\rangle\bigr)
+
o(\eta)
\le
c\,\eta\,\|g_t-g_t^s\|+o(\eta).
\]

Since the surrogate replaces the full MLP expansion \(Z\) by its projected version
\(
Z_s=P_S Z,
\)
its Jacobian with respect to the attacked hidden state is correspondingly projected:
\(
\frac{\partial Z_s}{\partial \widetilde{\mathbf H}}
=
P_S\frac{\partial Z}{\partial \widetilde{\mathbf H}}.
\)
Let
\[
r_t=\nabla_Z L_f\!\left(\widetilde{\mathbf H}^{(t)}\right),
\qquad
r_t^s=\nabla_{Z_s} L_s\!\left(\widetilde{\mathbf H}^{(t)}\right).
\]
Under the surrogate approximation, we identify the surrogate MLP-path gradient with the projected full gradient, i.e.
\(
r_t^s \approx P_S r_t.
\)
Hence
\[
g_t^s
=
\left(\frac{\partial Z_s}{\partial \widetilde{\mathbf H}}\right)^\top r_t^s
\approx
\left(\frac{\partial Z}{\partial \widetilde{\mathbf H}}\right)^\top P_S r_t,
\]
and therefore
\[
g_t-g_t^s
\approx
\left(\frac{\partial Z}{\partial \widetilde{\mathbf H}}\right)^\top (I-P_S)r_t.
\]
Accordingly, we use the bound
\[
\|g_t-g_t^s\|
\le
\left\|\frac{\partial Z}{\partial \widetilde{\mathbf H}}\right\|
\cdot
\|(I-P_S)r_t\|
+
\varepsilon_t,
\]
where \(\varepsilon_t\) collects the surrogate mismatch induced by replacing \(r_t^s\) with \(P_S r_t\). We work in the ideal projected-gradient case, i.e., $\varepsilon_t=0$, which we
include in the local regularity conditions of Theorem~\ref{thm:surrogate}.

For \(j\notin S\),
\[
|r_{t,j}^{(l_{\mathrm a})}|
=
\frac{s_{t,j}}{|a_{t,j}^{(l_{\mathrm a})}(x)|}
\le
\frac{s_{t,j}}{m_0},
\]
and thus
\[
\|(I-P_S)r_t\|^2
=
\sum_{j\notin S}\bigl(r_{t,j}^{(l_{\mathrm a})}\bigr)^2
\le
\frac{1}{m_0^2}\sum_{j\notin S}s_{t,j}^2
=
\frac{1}{m_0^2}M_t^2.
\]
Hence
\[
\|(I-P_S)r_t\|
\le
\frac{1}{m_0}M_t.
\]

Applying Lemma~\ref{lem:jacobian_z},
\[
\left\|\frac{\partial Z}{\partial \widetilde{\mathbf H}}\right\|
\le
L_\phi L_\sigma \|W_1\|
\left(
C_{\mathrm{attn}}\frac{B_QB_KB_V}{\sqrt{d_m}}
+
B_V
\right).
\]
Therefore
\[
G_t
\le
\frac{c\,L_\phi L_\sigma \|W_1\|}{m_0}
\left(
C_{\mathrm{attn}}\frac{B_QB_KB_V}{\sqrt{d_m}}
+
B_V
\right)\eta\,M_t
+
o(\eta).
\]

Summing over \(t\in[K]\), we obtain
\[
\sum_{t=1}^{K}G_t
\le
\frac{c\,L_\phi L_\sigma \|W_1\|}{m_0}
\left(
C_{\mathrm{attn}}\frac{B_QB_KB_V}{\sqrt{d_m}}
+
B_V
\right)\eta
\sum_{t=1}^{K}M_t
+
o(K\eta).
\]

By a union bound, the estimate holds with probability at least \(1-2Kp\). Denote 
\[
C_\mathrm{surr}=\frac{c\,L_\phi L_\sigma \|W_1\|}{m_0}
\left(
C_{\mathrm{attn}}\frac{B_QB_KB_V}{\sqrt{d_m}}
+
B_V
\right)
\] gives the final result.
\end{proof}

\subsection{A Spatial Transfer Limitation of Last-Token ReFT}
\label{app:generalize}

We further explain why a ReFT module trained only at the final token need not transfer optimally to earlier positions. The key point is that, at an earlier position, the reconstruction error naturally decomposes into two parts: one outside the learned edited subspace, and one inside that subspace.

\begin{proposition}[Orthogonal decomposition of transfer error under single-token LAT]
\label{prop:orthogonal_misalignment}
Let $(\mathbf R^*,\mathbf W^*,\mathbf b^*)$ be obtained by training ReFT only at the final token position $T$:
\[
(\mathbf R^*,\mathbf W^*,\mathbf b^*)
\in
\arg\min_{\mathbf R,\mathbf W,\mathbf b}
\;
\mathbb E_{\mathbf h_T}
\left[
\max_{\|\delta\|\le \varepsilon}
\|
\Phi_{\mathbf R}(\mathbf h_T+\delta)-\mathbf h_T
\|^2
\right].
\]
For an earlier position $t<T$, let $\tilde{\mathbf h}_t=\mathbf h_t+\Delta_t$. Then
\[
\mathbb E \Big[ \| \Phi_{\mathbf R^*}(\tilde{\mathbf h}_t)-\mathbf h_t \|^2 \Big]
=
\mathbb E \Big[ \| (\mathbf I-\mathbf R^{*\top}\mathbf R^*)\Delta_t \|^2 \Big]
+
\mathbb E \Big[ \| \mathbf R^{*\top}(\mathbf W^*\tilde{\mathbf h}_t+\mathbf b^* - \mathbf R^*\mathbf h_t) \|^2 \Big].
\]
\end{proposition}

\begin{proof}
By definition of the ReFT operator,
\[
\Phi_{\mathbf R^*}(\tilde{\mathbf h}_t)
=
(\mathbf I-\mathbf R^{*\top}\mathbf R^*)(\mathbf h_t+\Delta_t)
+
\mathbf R^{*\top}(\mathbf W^*\tilde{\mathbf h}_t+\mathbf b^*).
\]
Subtracting $\mathbf h_t$ and using
\[
\mathbf h_t
=
(\mathbf I-\mathbf R^{*\top}\mathbf R^*)\mathbf h_t
+
\mathbf R^{*\top}\mathbf R^*\mathbf h_t,
\]
we obtain
\[
\Phi_{\mathbf R^*}(\tilde{\mathbf h}_t)-\mathbf h_t
=
(\mathbf I-\mathbf R^{*\top}\mathbf R^*)\Delta_t
+
\mathbf R^{*\top}(\mathbf W^*\tilde{\mathbf h}_t+\mathbf b^*-\mathbf R^*\mathbf h_t).
\]
The first term lies in the orthogonal complement of the edited subspace, while the second lies inside the edited subspace. Since these two subspaces are orthogonal,
\[
\|\Phi_{\mathbf R^*}(\tilde{\mathbf h}_t)-\mathbf h_t\|^2
=
\|(\mathbf I-\mathbf R^{*\top}\mathbf R^*)\Delta_t\|^2
+
\|\mathbf R^{*\top}(\mathbf W^*\tilde{\mathbf h}_t+\mathbf b^*-\mathbf R^*\mathbf h_t)\|^2.
\]
Taking expectations yields the stated decomposition.
\end{proof}

The proposition highlights two distinct sources of transfer error when a ReFT module trained only at position $T$ is reused at an earlier position $t<T$. The first term on the right hand side measures perturbation components that lie outside the learned edited subspace and therefore cannot be removed by the reused low-rank intervention. The second term measures reconstruction mismatch within the edited subspace itself. Together, these terms show that even if the final-token ReFT module is effective at position $T$, its zero-shot reuse at earlier positions need not remain optimal. This provides additional support for applying the defense over a suffix window rather than only at the final token.

\section{Estimated Compute Calculations}
\label{app:computation}
\subsection{Estimated FLOPs}
We estimate adversarial-training compute using dense-equivalent Kaplan-style FLOPs. Let \(N\) denote the dense parameter count of the full model and \(D\) denote the number of processed tokens. Following standard scaling-law  approximations \citep{kaplan2020scaling}, we estimate one dense forward pass as
\begin{equation*}
C_{\mathrm{fwd}} \approx 2ND,
\end{equation*}
and one dense forward-backward training pass with parameter gradient accumulation as
\begin{equation*}
C_{\mathrm{fwd+bwd}} \approx 6ND.
\end{equation*}
When a backward pass is used only to obtain gradients with respect to an input or latent perturbation, and no parameter gradients are accumulated, we charge
\begin{equation*}
C_{\mathrm{input\text{-}bwd}} \approx 2ND.
\end{equation*}
Thus, one forward pass plus one input-gradient backward pass costs approximately \(4ND\).

In the main table, we report normalized FLOPs per 512-token training example. These estimates are intended to compare training-time computation across methods. Unless otherwise stated, we use dense model parameter counts even when the implementation uses quantization or parameter-efficient wrappers.

\paragraph{R2D2.}
R2D2 first generates an offline adversarial pool and then fine-tunes the full model on the resulting data. We therefore decompose the cost into a full-model fine-tuning term and an offline search term:
\begin{equation*}
C_{\mathrm{R2D2}} = C_{\mathrm{train}} + C_{\mathrm{search}}.
\end{equation*}
The training term uses full-model fine-tuning,
\begin{equation*}
C_{\mathrm{train}} \approx 6ND_{\mathrm{train}}.
\end{equation*}
For the offline attack-generation term, if generating an attacked example requires \(n_{\mathrm{fwd}}\) forward passes and \(n_{\mathrm{bwd}}\) input-gradient backward passes over \(D_{\mathrm{search}}\) tokens, we charge
\begin{equation*}
C_{\mathrm{search}}=(2n_{\mathrm{fwd}}+2n_{\mathrm{bwd}})N D_{\mathrm{search}},
\end{equation*}
summed over all generated adversarial examples. Since this search is performed offline, we amortize it over normalized 512-token training examples:
\begin{equation*}
C_{\mathrm{R2D2\ step}}
=
\frac{C_{\mathrm{train}}+C_{\mathrm{search}}}{D_{\mathrm{train}}/512}.
\end{equation*}

\paragraph{LAT.} 
Each inner PGD step uses one forward pass and one backward pass that computes gradients only with respect to the latent perturbation and stops at the attack layer $l_{\mathrm a}$; we charge the latter $2\rho_{l_{\mathrm a}}NT$, with $\rho_{l_{\mathrm a}}$ defined as for our method below. The outer update backpropagates both the clean and adversarial losses through the full model, costing $2\times 4NT$. Including the clean and adversarial forward passes and one shape-capture forward pass, the per-step cost with $k$ inner steps is
\begin{equation*}
C_{\mathrm{LAT}} = 2(k+3)NT + 2k\rho_{l_{\mathrm a}}NT + 8NT.
\end{equation*}

\paragraph{CAT.}
CAT performs a \(k\)-step continuous PGD inner attack and updates a LoRA
adapter. Each PGD step requires one dense forward pass and one input-gradient
backward pass through the dense model, so the inner-search cost is
\begin{equation*}
C_{\mathrm{search}} = 4kNT.
\end{equation*}
The subsequent training update uses one clean forward pass and one adversarial
forward pass through the dense model. Although only the LoRA adapter parameters
are updated, gradients must still be backpropagated through the frozen dense
network to reach the adapter modules. We therefore estimate the training-update
cost as
\begin{equation*}
C_{\mathrm{train}}
\approx
4NT + 4NT + 2N_{\mathrm{LoRA}}T = (8 + 2p_{\mathrm{LoRA}})NT,
\end{equation*}
where \(N_{\mathrm{LoRA}}\) is the number of trainable LoRA parameters and
\(p_{\mathrm{LoRA}} = N_{\mathrm{LoRA}}/N\). In our runs,
\(p_{\mathrm{LoRA}} < 1\%\), so this term is numerically close to \(8NT\).
Thus,
\begin{equation*}
C_{\mathrm{CAT}}
\approx
4kNT + (8 + 2p_{\mathrm{LoRA}})NT
\approx
(4k+8)NT.
\end{equation*}

\paragraph{Ours.}
Our method uses a pruned surrogate for the inner attack search and updates only a small ReFT intervention on the target model. Let \(N_{\mathrm{s}}\) denote the dense-equivalent surrogate parameter count and \(N_{\mathrm{tr}}\) denote the number of trainable intervention parameters. If the ReFT intervention is applied at layer \(l_{\mathrm r}\), the latent attack perturbation is applied at layer \(l_{\mathrm a}\), and the model has \(L_{\mathrm{all}}\) layers, define
\begin{equation*}
\rho_{l_{\mathrm r}}
=
\frac{L_{\mathrm{all}}-l_{\mathrm r}-1}{L_{\mathrm{all}}},
\qquad
\rho_{l_{\mathrm a}}
=
\frac{L_{\mathrm{all}}-l_{\mathrm a}-1}{L_{\mathrm{all}}}.
\end{equation*}
Here \(\rho_{l_{\mathrm r}}\) and \(\rho_{l_{\mathrm a}}\) approximate the fraction of layers traversed by backward computation after the intervention or attack layer.

The target-model update requires a clean forward pass, an adversarial forward pass, and a backward pass from the loss to the ReFT intervention. We estimate this cost as
\begin{equation*}
C_{\mathrm{target}}
\approx
4NT + 4\rho_{l_{\mathrm r}}NT + 4N_{\mathrm{tr}}T.
\end{equation*}
The final term accounts for gradient accumulation on the trainable intervention parameters and is typically negligible relative to dense model computation.

For the surrogate inner search, each PGD step uses a surrogate forward pass and a backward pass from the output to the latent attack layer. Therefore,
\begin{equation*}
C_{\mathrm{surrogate}}
\approx
2k(1+\rho_{l_{\mathrm a}})N_{\mathrm{s}}T.
\end{equation*}
The total per-step cost of our method is
\begin{equation*}
C_{\mathrm{ours}}
=
C_{\mathrm{target}}+C_{\mathrm{surrogate}}.
\end{equation*}

For surrogate runs where the exact dense-equivalent surrogate parameter count is not directly available, we approximate
\begin{equation*}
N_{\mathrm{s}}
\approx
N\left(1 - f_{\mathrm{MLP}}r_{\mathrm{p}}\right),
\end{equation*}
where \(f_{\mathrm{MLP}}\) is the fraction of dense model parameters in MLP blocks and \(r_{\mathrm{p}}\) is the pruning ratio of MLP neurons. 

\subsection{Measured Compute Time and Peak Memory}
\label{app:wallclock}

We additionally report implementation-level timing of the dominant adversarial training computation and the peak GPU memory of the training process. All measurements use NVIDIA H200 NVL GPUs, bfloat16 precision, a maximum sequence length of 512, and eight inner PGD/search steps. The per-device batch size is 8 for Pythia-1.4B and 4 for Qwen2.5-3B. Timing is averaged over 100 measured steps for Pythia-1.4B and 50 measured steps for Qwen2.5-3B after excluding five warm-up steps. 

\begin{table}[ht]
\centering
\caption{Measured implementation-level compute time and peak training memory under matched configurations. Timing follows the native training-computation region instrumented for each implementation.}
\label{tab:wallclock}
\small
\begin{tabular}{llcc}
\toprule
Model & Method & Compute time/step (s)$\downarrow$
& Peak memory (GB)$\downarrow$ \\
\midrule
\multirow{3}{*}{Pythia-1.4B}
& CAT      & 0.7908 & 25.29 \\
& LAT      & 0.6517 & 25.39 \\
& LAT-ReFT & \textbf{0.4442} & \textbf{11.01} \\
\midrule
\multirow{3}{*}{Qwen2.5-3B}
& CAT      & 0.8613 & 51.18 \\
& LAT      & 0.6469 & 56.44 \\
& LAT-ReFT & \textbf{0.5650} & \textbf{14.77} \\
\bottomrule
\end{tabular}
\end{table}

In these implementation-level measurements, the recorded compute time of LAT-ReFT is 31.8\% lower than that of LAT on Pythia-1.4B, closely matching the 31.5\% reduction in estimated FLOPs, and 12.7\% lower on Qwen2.5-3B. 
LAT-ReFT also reduces peak GPU memory by 56.6\% and 73.8\%, respectively. On Qwen2.5-3B, the measured time reduction is smaller than the estimated FLOPs reduction, which we attribute to fixed and implementation-dependent costs, such as kernel launches, synchronization, and memory operations.

\end{document}